\documentclass[twoside]{article}

\usepackage[accepted]{aistats2025}

\usepackage[round]{natbib}
\usepackage{floatrow}
\newfloatcommand{capbtabbox}{table}[][\FBwidth]
\renewcommand{\bibname}{References}
\renewcommand{\bibsection}{\subsubsection*{\bibname}}

\usepackage[utf8]{inputenc} 
\usepackage[T1]{fontenc}    

\usepackage[colorlinks,citecolor=black,linkcolor=black,urlcolor=black]{hyperref}       
\usepackage{url}            
\usepackage{booktabs}       
\usepackage{amsfonts}       
\usepackage{nicefrac}       
\usepackage{microtype}      
\usepackage{xcolor}         
\usepackage{cuted}

\usepackage{microtype}
\usepackage{graphicx}
\usepackage{subcaption}
\usepackage{caption}
\usepackage{placeins}
\usepackage{makecell}

\usepackage{amsmath}
\usepackage{amssymb}
\usepackage{mathtools}
\usepackage{amsthm}
\usepackage{multirow}
\usepackage{arydshln}
\usepackage{xfrac}

\usepackage{color}

\usepackage{blindtext}
\usepackage{enumerate}
\usepackage{graphicx}
\usepackage{amsfonts, amsmath, bm, amssymb}

\usepackage{dsfont}
\usepackage{wrapfig}
\usepackage{subcaption}
\newcommand\numberthis{\addtocounter{equation}{1}\tag{\theequation}}
\usepackage{aligned-overset}

\usepackage{pifont}
\newcommand{\cmark}{\ding{51}}%
\newcommand{\xmark}{\ding{55}}%

\newcommand{\zero}{\mathbf{0}}

\newcommand{\diff}{\,\mathrm{d}}

\newcommand{\cev}[1]{\reflectbox{\ensuremath{\vec{\reflectbox{\ensuremath{#1}}}}}}

\newcommand{\inner}[2]{\left\langle #1, #2 \right\rangle}

\usepackage{xspace}
\makeatletter
\DeclareRobustCommand\onedot{\futurelet\@let@token\@onedot}
\def\@onedot{\ifx\@let@token.\else.\null\fi\xspace}
 
\def\ie{\emph{i.e}\onedot}

\def\wrt{w.r.t\onedot} 

\makeatother

\newcommand{\Dc}{\mathcal{D}}

\newcommand{\Gc}{\mathcal{G}}

\newcommand{\Lc}{\mathcal{L}}

\newcommand{\Nc}{\mathcal{N}}

\newcommand{\Rc}{\mathcal{R}}

\newcommand{\Xc}{\mathcal{X}}

\newcommand{\ds}{\mathsf{d}}
\newcommand{\es}{\mathsf{e}}
\newcommand{\fs}{\mathsf{f}}

\newcommand{\hs}{\mathsf{h}}

\newcommand{\ks}{\mathsf{k}}

\newcommand{\rs}{\mathsf{r}}

\newcommand{\ws}{\mathsf{w}}
\newcommand{\xs}{\mathsf{x}}

\newcommand{\Eb}{\mathbb{E}}

\newcommand{\Rb}{\mathbb{R}}

\newcommand{\av}{\mathbf{a}}
\newcommand{\bv}{\mathbf{b}}

\newcommand{\ev}{\mathbf{e}}
\newcommand{\fv}{\mathbf{f}}
\newcommand{\gv}{\mathbf{g}}
\newcommand{\hv}{\mathbf{h}}

\newcommand{\mv}{\mathbf{m}}

\newcommand{\rv}{\mathbf{r}}
\newcommand{\sv}{\mathbf{s}}

\newcommand{\uv}{\mathbf{u}}

\newcommand{\xv}{\mathbf{x}}
\newcommand{\yv}{\mathbf{y}}

\newcommand{\Av}{\mathbf{A}}

\newcommand{\Gv}{\mathbf{G}}

\newcommand{\Iv}{\mathbf{I}}

\newcommand{\epsilonv   }{\boldsymbol \epsilon   }

\newcommand{\thetav     }{\boldsymbol \theta     }

\newcommand{\kappav     }{\boldsymbol \kappa     }

\newcommand{\muv        }{\boldsymbol \mu        }

\ifx\BlackBox\undefined
\newcommand{\BlackBox}{\rule{1.5ex}{1.5ex}}  
\fi
\ifx\QED\undefined
\def\QED{~\rule[-1pt]{5pt}{5pt}\par\medskip}
\fi
\ifx\proof\undefined
\newenvironment{proof}{\par\noindent{\em Proof:\ }}{\hfill\BlackBox\\}
\fi
\ifx\theorem\undefined
\newtheorem{theorem}{Theorem}
\fi
\ifx\example\undefined
\newtheorem{example}{Example}
\fi
\ifx\property\undefined
\newtheorem{property}{Property}
\fi
\ifx\lemma\undefined
\newtheorem{lemma}{Lemma}
\fi
\ifx\proposition\undefined
\newtheorem{proposition}{Proposition}
\fi
\ifx\fact\undefined
\newtheorem{fact}{Fact}
\fi
\ifx\remark\undefined
\newtheorem{remark}{Remark}
\fi
\ifx\corollary\undefined
\newtheorem{corollary}{Corollary}
\fi
\ifx\definition\undefined
\newtheorem{definition}{Definition}
\fi
\ifx\conjecture\undefined
\newtheorem{conjecture}{Conjecture}
\fi
\ifx\axiom\undefined
\newtheorem{axiom}[theorem]{Axiom}
\fi
\ifx\claim\undefined
\newtheorem{claim}[theorem]{Claim}
\fi
\ifx\assumption\undefined
\newtheorem{assumption}{Assumption}
\fi
\ifx\question\undefined
\newtheorem{question}{Question}
\fi
\ifx\problem\undefined
\newtheorem{problem}{Problem}
\fi

\usepackage[capitalize,noabbrev]{cleveref}
\usepackage{thm-restate}
\crefformat{prop}{Prop~#2#1#3}
\crefformat{figure}{Fig~#2#1#3}
\crefformat{lem}{Lem~#2#1#3}
\crefformat{lemma}{Lem~#2#1#3}
\crefformat{section}{Sec~#2#1#3}
\crefformat{equation}{(#2#1#3)}
\crefformat{chapter}{Chapter~#2#1#3}
\crefalias{lem}{Lemma}
\crefformat{corollary}{Cor~#2#1#3}
\crefformat{proposition}{Prop~#2#1#3}
\Crefname{appendix}{Appx}{Appx}

\begin{document}

\twocolumn[
\aistatstitle{Diffusion Models under Group Transformations}
\aistatsauthor{ Haoye Lu$^*$ \And Spencer Szabados$^*$ \And Yaoliang Yu}
\aistatsaddress{University of Waterloo, Vector Institute } ]
\def\thefootnote{*}\footnotetext{Equal contribution. Haoye proposed the theoretical results, while both Spencer and Haoye dedicated substantial time to model implementation and evaluation. Yaoliang supervised the project and provided valuable guidance.}\def\thefootnote{\arabic{footnote}}
\begin{abstract}
In recent years, diffusion models have become the leading approach for distribution learning. This paper focuses on structure-preserving diffusion models (SPDM), a specific subset of diffusion processes tailored for distributions with inherent structures, such as group symmetries. We complement existing sufficient conditions for constructing SPDMs by proving complementary necessary ones. Additionally, we propose a new framework that considers the geometric structures affecting the diffusion process. Leveraging this framework, we design a structure-preserving bridge model that maintains alignment between the model’s endpoint couplings. Empirical evaluations on equivariant diffusion models demonstrate their effectiveness in learning symmetric distributions and modeling transitions between them. Experiments on real-world medical images confirm that our models preserve equivariance while maintaining high sample quality. We also showcase the practical utility of our framework by implementing an equivariant denoising diffusion bridge model, which achieves reliable equivariant image noise reduction and style transfer, irrespective of prior knowledge of image orientation.
\end{abstract}

\section{INTRODUCTION}
\label{sec:intro}

Diffusion models \citep{SongE2019, HoJA2020, SongME2021, SongDKKEP2021, RombachBLEO2022, KarrasAAL2022, SongDCS2023} have become the leading method in a plethora of generative modelling tasks including image generation \citep{SongE2019, HoJA2020, SongME2021}, audio synthesis \citep{KongPHZC2021}, image segmentation \citep{BaranchukVRKB2022, WollebSBFPP2022}, image editing and style transfer \citep{MengHSSWZE2022, ZhouLKE2024}.

In many generation tasks, the data involved often exhibit inherent ``structures'' that are invariant -- or the mappings between them being equivariant -- under a set of transformations. For example, it is commonly assumed that the distribution of photographic images is invariant under horizontal flipping. In tasks such as image denoising or inpainting, where the orientation of an image is not provided, it is natural to require the denoised or inpainted image to retain the same orientation as the input. Namely, the denoising or inpainting processes should exhibit equivariance under rotations and flipping (see \cref{fig:image-lysto-denosing}).

In critical applications, such as medical imaging analysis, these properties must not only be desired but also theoretically guaranteed to ensure model output consistency and prevent the introduction of biases or errors. Diagnostic images, such as X-rays or biopsy assays, are often captured from different orientations \citep{LafargeBPDV2021, ShaoLWJZ2023}, which has led many diagnostic methods to be designed with invariance to image transformations. These methods are used in tasks such as distinguishing between benign and malignant breast lesions in digitized mammograms \citep{PohlmanPOC1996, RangayyanEDA1997}, analyzing blood cells \citep{LinXM1998}, and performing digital pathology segmentation \citep{VeelingLWCW2018}. Since these methods often require high-resolution images for precise segmentation \citep{PohlmanPOC1996}, denoising techniques are applied to improve image quality. For these techniques to function effectively, they must exhibit perfect equivariance—without it, the overall method’s invariance may be compromised.

\begin{figure}
    \includegraphics[width=0.95\columnwidth]{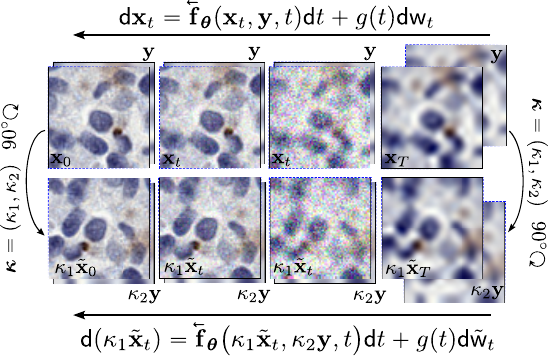}
    \caption{Equivariant inference trajectory: if $(\xv_T, \yv)$ undergos a $90^\circ$ rotation (via operator $\kappav$), the denoised output precisely mirrors this rotation.}\label{fig:image-lysto-denosing}
\end{figure}

This paper investigates structure-preserving diffusion models (SPDM), a family of diffusion processes that preserve the group-invariant properties of the distributions. Our framework extends previous research \citep{YimTDMDBJ2023,XuYSSET2022,HoogeboomSVVW2022, QiangSXGGZML2023, MartinkusLLVHRWCBGL2023} on drift equivariance by incorporating additional factors influencing diffusion trajectories, not just noisy samples; allowing us to characterize structure preserving properties beyond those developed for classical diffusion processes, see \cref{sec:group_inv_sde}. \cref{fig:image-lysto-denosing} illustrates an example of this, for denoising diffusion bridge models (DDBM)~\citep{ZhouLKE2024, BortoliLCTN2023}, where variables $\yv$ correspond to starting point $\xv_T$ of the backward sampling process. The extension allows us to build a bridge model that captures the equivariant coupling precisely between $\xv_T$ and $\xv_0$ where a rotation on $\xv_T$ results in the same rotation on $\xv_0$.

Building upon this generalized framework, we establish an equivalence relationship between the group equivariance of the drifts and the structural preservation of the distributions of induced flows. This development complements existing discussions around structural diffusion models, which have predominantly focused on identifying sufficient conditions. We exemplify the utility of our framework by presenting two equivariant score-based models that achieve theoretically guaranteed capabilities for invariant data generation and equivariant data editing: (1) SPDM+WT, using a weight-tied implementation to reduce training and sampling costs at the expense of image quality, and (2) SPDM+FA, employing Framing-Averaging \citep{PunyASMGHL2022} to combine outputs from conventionally trained diffusion models, attaining the same theoretical guarantees while achieving a sample quality comparable, or superior, to standard diffusion models. 

Unlike other equivariant implementations of diffusion models that incorporate FA during training \citep{MartinkusLLVHRWCBGL2023, DuvalSHMMBR2023}, our method applies FA only during inference, greatly reducing training costs. Empirical studies on both artificial and medical image datasets support our claims. Additionally, we demonstrate the effectiveness of our method for equivariant denoising and style transfer, as shown in~\cref{fig:image-lysto-denosing}.

Our code is available at \url{https://github.com/watml/SPDM}.

\section{RELATED WORK}
\label{sec:related_work}

The problem of conditioning neural networks to respect group-invariant (or equivariant) distributions \citep{ShaweTaylor1993} a longstanding issue within the domains of physical modelling, computer vision, and, generative modelling. This is underscored by the widespread utilization of diverse forms of data augmentation, that seek to increase model robustness to perturbation. However, achieving true group invariance (or equivariance) through data augmentation alone is impractical, as it would require an infinite number of samples to guarantee invariance. Consequently, models conditioned solely by data augmentation often fail to fully capture the desired properties \citep{ElesedyZ2021, GaoDLL2022}. As our primary topic is diffusion models, we will devote this section to these works.

The study of group invariance within diffusion models \citep{SongE2019, SongME2021,HoJA2020,KarrasAAL2022,KimLLMTUHME2023,YimTDMDBJ2023}, and diffusion bridge models \citep{BortoliTHD2021, LiuVHTNA2023, ZhouLKE2024, LeeLBK2024}, has focused primarily on applications in molecule generation (e.g., molecular conformation, and protein backbone generation) \citep{ShiLXT2021,XuYSSET2022,HoogeboomSVVW2022,YimTDMDBJ2023,JingCCBJ2022,CorsoSJBJ2023,MartinkusLLVHRWCBGL2023}. Most of these approaches can be broadly described as conditioning the diffusion process on a graph prior that represents the unconformed molecule and employing a transformation (applied to the inner molecular atomic distances - such as the relative torsion angle coordinates \citep{JingCCBJ2022}) that produces a group-invariant form (or one that is more robust to the selected group transformations). This thereby, results in a representation that is sufficient to ensure the diffusion process is equivariant. More generally,  \citet{BortoliMH2022, Mathieu2023}, and \citet{YimTDMDBJ2023}, investigate distribution invariance over more general geometries (e.g., Riemannian manifolds generated by Lie groups). The study of distribution invariance comes about naturally as a result of finding a limiting probability distribution over the geometry in these settings, a requirement for the diffusion process to be well-defined. These methods, which are most similar to ours, operate by designing Gaussian kernels, those that define the diffusion process, that are invariant to select linear isometry groups.

In this work we extend existing theoretical results, developed within the forgoing works, by providing a complete characterization of the necessary and sufficient conditions of the drift and diffusion terms in order to ensure a diffusion process is invariant under linear isometry groups. The diffusion processes we consider encapsulate both regular diffusion models and diffusion bridge models with and without conditioning, treating the conditioning variable as a tensor admitting its own group operations. Our work is expressly more general than existing results which focus primarily on providing sufficient conditions with a single conditioning variable.

\section{PRELIMINARY}
\label{sec:preliminary}


\subsection{Diffusion Processes and Diffusion Bridges}
\label{sec:prel_diff_bridge}
Let $\{\xv_t\}_{t=0}^T$ denote a set of time-indexed random variables in ${\Rb^d}$ such that $\xv_t \sim p_t$, where $p_t$  are the marginal distributions of an underlying diffusion process defined by a stochastic differential equation (SDE):
\begin{equation}
	\diff \xv_t = \uv(\xv_t, t)\diff t + g(t) \diff \ws_t, ~~ \xv_0\sim p_0(\xv_0). \label{eq:fwd_diff_process}
\end{equation}
Here, $\uv: \Rb^d \times [0,T] \rightarrow \Rb^d$ is the \textit{drift}, $g: [0,T] \rightarrow \Rb$ is a scalar \textit{diffusion coefficient}, and $\ws_t \in \Rb^d$ denotes a Wiener process. In generative diffusion models, we take $p_0 = p_{\rm data}$ and $p_T = p_{\rm prior}$; thereby, the diffusion process constructs a path $p_t$ from $p_{\rm data}$ to $p_{\rm prior}$.  

In practice, $\uv$ and $g$ are chosen to accelerate the sampling of $\xv_t$ in \cref{eq:fwd_diff_process}. Table~\ref{tb:fg_selection} in  \cref{appx:diffusion_coefficients} lists two popular choices of $\uv$ and $g$, corresponding to the variance preserving~(VP, \citealt{HoJA2020, SongME2021}) and variance exploding~(VE, \citealt{SongDKKEP2021}) SDEs.

For these selections, $p(\xv_t \vert \xv_0)$ is an easy-to-sample spherical Gaussian, and the sampling of $\xv_t$ is carried out by first picking $\xv_0 \sim p_0$ and then sampling from $p(\xv_t \vert \xv_0)$. 

For any SDE, there is a corresponding reverse SDE with the same marginal distribution $p_t$ for all $t\in [0,T]$ \citep{anderson1982}. In fact, the forward SDE in \cref{eq:fwd_diff_process} has a family of reverse-time SDEs \citep{ZhangC2023}:
\begin{equation}
\begin{aligned}
    \diff \xv_t =
    & \Big[\uv(\xv_t, t) - \tfrac{1+\lambda^2}{2} g^2(t)  \nabla_{\xv_t} \log p_t(\xv_t)\Big] \diff t\\ 
    &+ \lambda g(t) \diff \ws_t, ~~ \text{for all}~\lambda \geq 0. 
\end{aligned}
\label{eq:general_reverse_SDE_family}
\end{equation}
Setting $\lambda = 1$ in \cref{eq:general_reverse_SDE_family} simplifies the equation to the original reverse SDE as derived in \citep{anderson1982}. For $\lambda = 0$, the process transforms into a deterministic ODE process, known as the probability flow ODE (PF-ODE, \citealt{SongDKKEP2021}). The only unknown term is the score $\nabla_\xv \log p_t(\xv)$, which is  estimated using a neural network $\sv_{\thetav}(\xv, t)$ trained through score matching \citep{SongDKKEP2021} or an equivalent denoising loss \citep{KarrasAAL2022}. Subsequently, after training, one can sample $\xv_0 \sim p_{\rm data}(\xv_0)$ by solving the SDE (or ODE if $\lambda = 0$) given in \cref{eq:general_reverse_SDE_family} starting from $\xv_T \sim p_T(\xv_T)$. 

Instead of building path $p_t$ from a data distribution, $q_{\rm data}(\xv, \yv)$, to a known prior, diffusion bridges~(DBs) create path $q_t$ such that $q_0(\xv) = q_{\rm data}(\xv)$ and $q_T(\yv) =  q_{\rm data}(\yv)$. For $(\xv, \yv) \sim  q_{\rm data}(\xv, \yv)$, DBs leverage the distribution $p_t$ induced by \cref{eq:fwd_diff_process} with $\xv_0 = \xv$ and $\xv_T = \yv$ to sample $\xv_t$. In this way, 
\(
    q_t(\xv_t) = {\Eb}_{(\xv, \yv) \sim  q_{\rm data}(\xv, \yv)}\big[ p_t(\xv_t \vert \xv_0 = \xv, \xv_T = \yv) \big]
\) and admits the forward SDE:
\begin{equation}
\label{eq:ddbm_fwd}
    \diff \xv_t = [\uv(\xv_t, t)+ g(t)^2 \hv(\xv,\xv_T, t)] \diff t + g(t) \diff \ws_t,
\end{equation}
given $\xv_T = \yv$ and $\hv(\xv,\xv_T, t) = \nabla_{\xv} \log p(\xv_T \vert \xv)$, the gradient of the log transition kernel induced by \cref{eq:fwd_diff_process}. For $\uv$ and $g$ in \cref{tb:fg_selection}, $p_t(\xv_t \vert \xv_0, \xv_T)$ can be sampled efficiently. Thus, $\xv_t \sim q_t(\xv_t)$ can be obtained by first sampling $(\xv, \yv)$ and then $\xv_t$. 

DBs can be broadly categorized into those that condition on $\xv_T$ and those that do not. \citet{BortoliLCTN2023} show that conditioning is necessary for effectively learning the coupling encoded in $q_{\rm data}$. This is crucial for many practical tasks, such as image denoising, where the denoised image should match the blurry input. For the conditioned DBs \citep{ZhouLKE2024}, the family of the backward SDE, conditioned on $\xv_T = \yv$, is
\begin{equation}
\begin{aligned}
    \diff \xv_t =
    & [\uv(\xv_t, t) + g(t)^2 \hv(\xv,\xv_T, t)\\ 
    &- \tfrac{1+\tau^2}{2}g(t)^2 \sv(\xv_t \vert \xv_T, t)]\diff t
    + \tau g(t) \diff \ws_t,
\end{aligned}
\label{eq:ddbm_bwd}
\end{equation}
for all $\tau \geq 0$, where $\sv(\xv_t \vert \xv_T, t) = \nabla_{\xv_t}\log q_t(\xv_t \vert \xv_T)$ is the score of the DB's distribution $q_t$ given that $\xv_T = \yv$. 

Notably, besides the noisy sample $\xv_t$, the drift terms in \eqref{eq:ddbm_fwd} and \eqref{eq:ddbm_bwd} also depend on $\yv = \xv_T$. This dependency leads us to consider a more general diffusion process:
\begin{align}
    \diff \xv_t = \fv(\xv_t, \yv, t)\diff t + g(t) \diff \ws_t ,\quad \xv_0 \sim p(\xv_0 \vert \yv)\label{eq:fwd_diff_process_general},
\end{align}
where $\yv$ represents other factors affecting the process and does not need to have the same shape as $\xv_t$. When $\yv$ consists of zero entries, \eqref{eq:fwd_diff_process_general} simplifies to the standard diffusion process in \eqref{eq:fwd_diff_process}. Currently, all structure-preserving diffusion frameworks are based on \eqref{eq:fwd_diff_process}.
\subsection{Group Invariance and Equivariance}
\label{sec:group-invariance}
A set of functions $\Gc  = \{\kappa: \Xc \rightarrow \Xc\}$ equipped with an associative binary operation $\circ:\Gc\times\Gc\to \Gc$, composition in this case, is called a \textit{group} if (1) for any $\kappa_1, \kappa_2 \in \Gc$, $\kappa_1 \circ \kappa_2\in \Gc$; (2) $\Gc$ has an identity operator $\ev$ with $\ev\circ\kappa=\kappa\circ\ev=\kappa$; and (3) for any $\kappa \in \Gc$, there exists an inverse operator $\kappa^{-1}$ such that $\kappa^{-1} \circ \kappa = \kappa \circ \kappa^{-1} = \ev$. For instance, let $\fs_x$ be the operator that flips images horizontally, then $\Gc = \{\fs_x, \ev\}$ is a group as $\fs_x^{-1} = \fs_x$. 

In this paper, we limit our focus to $\kappa$ that do not alter Wiener processes. From now on we assume  $\Gc$ consists of isometries $\kappa$ that fix zero. That is, $\kappa$ is a distance-preserving transformation such that $\|\kappa\xv\|_2 = \|\xv\|_2$ and $\|\kappa\xv - \kappa \yv\|_2 = \|\xv - \yv\|$ for all $\xv,\yv \in \Rb^d$; see \cref{appx:isometry} for details. Consequently, any $\kappa \in \Gc$ can be expressed as an orthogonal matrix $A_\kappa \in \Rb^{d \times d}$, where $\kappa \xv = A_\kappa \xv$.

For a distribution with density $p$ defined on $\Rb^d$, we say $p$ is \emph{$\Gc$-invariant} if $p(\xv) = p(\kappa \xv)$ for all $\kappa \in \Gc$. Likewise, for a conditional distribution $p(\xv\vert\yv)$ with $\xv\in \Rb^m$ and $\yv\in \Rb^n$, we say it is \emph{$\Gv$-invariant} if $p(\kappa_1 \xv \vert \kappa_2 \yv) = p(\xv \vert \yv)$ for all $(\kappa_1, \kappa_2) \in \Gv$, $\xv\in \Rb^m$ and $\yv\in \Rb^n$. As shorthand, we will write $\Gv = \{\kappav = (\kappa_1, \kappa_2) | \kappa_1: \Rb^m \rightarrow \Rb^m,  \kappa_2: \Rb^n \rightarrow \Rb^n\}$  defined in $\Rb^{m+n}$ such that $\kappav (\xv, \yv) = (\kappa_1 \xv, \kappa_2 \yv) = (A_{\kappa_1}  \xv, A_{\kappa_2} \yv)$ with orthogonal $A_{\kappa_1} \in \Rb^{m\times m}$ and $A_{\kappa_2} \in \Rb^{n\times n}$. 

Since diffusion models learn a distribution by estimating its score, the following lemma demonstrates that if a distribution is $\Gc$-invariant, its score must be \emph{$\Gc$-equivariant}. We also independently present a special case of this result, excluding the condition on $\yv$, which aligns with the findings in existing structure-preserving diffusion frameworks \citep{HoogeboomSVVW2022,Mathieu2023}. Proofs are provided in \cref{appx:proofs}.
\begin{restatable}{lemma}{eqvScore}
\label{lem:eqv_score}
    $p(\xv\vert\yv)$ is $\Gv$-invariant if and only if $\sv(\kappa_1\xv\vert \kappa_2 \yv) = \kappa_1 \circ \sv(\xv\vert \yv)$ for all $(\kappa_1, \kappa_2) \in \Gv$, $\xv\in \Rb^m$ and $\yv\in \Rb^n$. Likewise, $p(\xv)$ is $\Gc$-invariant if and only if $\sv(\kappa\xv) = \kappa  \circ  \sv(\xv)$ for all $\kappa \in \Gc$.
\end{restatable}    

We visualize the meaning of \cref{lem:eqv_score} in \cref{fig:eqv_score_func} by illustrating an equivariant (unconditional) distribution under diagonal flipping about the line $x=y$. As can be seen, symmetric points across the diagonal have symmetric scores, showing that flipping a point before or after score evaluation yields the same result.

\section{STRUCTURE PRESERVING PROCESSES}
\label{sec:group_inv_sde}

In this section, we discuss diffusion processes that maintain the structure of the data distribution. Specifically, given a data distribution that is assumed to be group invariant, we explore the sufficient and necessary configurations of diffusion processes to preserve this invariance throughout the diffusion trajectory. These insights will serve as a foundation for the design of the equivariant neural networks shown in \cref{sec:model-architectures}.
We present our main theoretical results in \cref{prop:eqv_drift}:

\begin{restatable}{proposition}{eqvDrift}
\label{prop:eqv_drift}
    Given a diffusion process in \eqref{eq:fwd_diff_process_general} with $\Gv$-invariant $p_0(\xv_0 \vert \yv)$, let $[\zero]_{p_t}$ be the set of ODE drifts preserving the distribution $p_t$. Then $p_t(\xv_t \vert \yv)$ is $\Gv$-invariant for all $t\geq 0$  if and only if
    \begin{equation}
        \kappa_1^{-1} \circ \fv(\kappa_1 \xv, \kappa_2 \yv, t) - \fv(\xv, \yv, t) \in [\zero]_{p_t} \label{eq:eqv_drift_diff}
    \end{equation}
    for all $t > 0$, $\xv \in \Rb^m$, $\yv \in \Rb^n$ and $\kappav \in \Gv$.
\end{restatable}

While \cref{prop:eqv_drift} is presented based on the conditional distribution \( p_t(\xv_t \vert \yv) \), it also applies to the unconditional case by setting \( n = 0 \). In this scenario, \( p_t(\xv_t \vert \yv) \) reduces to \( p_t(\xv_t) \), \(\Gv\) simplifies to \(\Gc\) consisting of \(\kappa_1\), and \cref{eq:eqv_drift_diff} becomes \( \kappa_1^{-1} \circ \fv(\kappa_1 \xv, t) - \fv(\xv, t) \in [\zero]_{p_t} \). Existing structure-preserving diffusion models \citep{YimTDMDBJ2023, XuYSSET2022, HoogeboomSVVW2022, QiangSXGGZML2023, MartinkusLLVHRWCBGL2023} are based on the special case that \( \kappa_1^{-1} \circ \fv(\kappa_1 \xv, t) - \fv(\xv, t) = \zero \). However, this zero-drift condition is not the only one that preserves $p_t$. For instance, a spherical Gaussian can be preserved by any circular vector field, where the drift is aligned with the boundary of the level set of the density function.

To gain an intuitive understanding of why \cref{eq:eqv_drift_diff} results in the \(\Gv\)-invariance of \( p_t(\xv \vert \yv) \), we illustrate how the proposition works in DB models for $\Gv = \{\kappav = (\kappa, \kappa) \vert \kappa\in \Gc\}$. For simplicity, we assume \(\fv(\kappa \xv, \kappa \yv, t) - \kappa \circ \fv(\xv, \yv, t) = \zero\), which implies $\fv$ is equivariant. As discussed in \cref{sec:prel_diff_bridge}, in DBs, \(\yv\) corresponds to the starting point \(\xv_T\) of the backward process, where \(\xv_T\) can be intuitively understood as a noisy image and \(\xv_0\) as the corresponding denoised version. When $\fv$ is equivariant, it essentially says for an infinitesimal step $\delta$, the transition probability induced by the SDE in \cref{eq:fwd_diff_process_general} satisfies
\begin{align}
    p(\av \vert  \bv, \xv_T) = p(\kappa \av \vert \kappa\bv, \kappa\xv_T).
\end{align}
As $p_T(\xv_T \vert \xv_T) = p_T(\kappa\xv_T \vert \kappa\xv_T)$, which is $\Gv$-invariant, applying this relationship recursively from $t = T$ to $0$ implies $p(\xv_t \vert \xv_T) = p(\kappa\xv_t \vert \kappa \xv_T)$ for $t \in [0, T]$. (Since the SDE is solved reversely, the base case becomes the invariance of $p_T$ instead of $p_0$) Intuitively, suppose $\kappa$ denotes the image flipping operator. This basically says, when input blurry image $\xv_T$ is flipped, so is the denoised image $\xv_0$. In \cref{fig:eqv_sde_traj} (\textbf{Left}), we visualize the evolution of conditional $p_t$ when the conditioned end point $\xv_T$ is flipped \wrt $x = 0$. As we can see when $\xv_T$ is flipped, the trajectory from $\xv_T$ to $\xv_t$ is also flipped, which corresponds to the invariance of $p(\xv_t \vert \xv_T)$. 

For completeness, in \cref{fig:eqv_sde_traj} (\textbf{Right}), we also present an example of unconditional $p_t$ in \cref{prop:eqv_drift} by setting $n = 0$. In this case, to ensure that $p_t$ is invariant, it suffices that \( \fv(\kappa_1 \xv, t) - \kappa_1\circ \fv(\xv, t) = \zero \). Here, we visualize the evolution of $p_t$ driven by two different diffusion processes with $p_0$ invariant to flipping with respect to $x=0$ (that is, $\kappa_1(x)=-x$). In the upper plot, we have drift $\fv(x, t) = \frac{1 - x}{1-t}$ that pushes $x$ to $1$ and is not flip-equivariant for $t \geq 0$. As we observe, for all $t > 0$, $p_t$ is no longer flip-invariant, which corroborates \cref{prop:eqv_drift}. In contrast, the lower plot illustrates the VP-SDE (see \cref{tb:fg_selection}) with $\alpha_t = 1 - t$ for $t \in (0,1)$. The drift here, $\fv(x, t) = - \frac{x}{2(1-t)}$ is  is flip-equivariant as $\fv(-x, t) = - \fv(x, t)$. As shown, $p_t$ has a symmetric density for all $t \geq 0$, which is also aligned with \cref{prop:eqv_drift}. 

In summary, \cref{prop:eqv_drift} shows for conditional \( p_t \), \cref{eq:eqv_drift_diff} ensures the coupling relationship between condition \(\yv\) and the noise sample \(\xv_t\) is \(\Gv\)-invariant. In contrast, for unconditional \( p_t \), it ensures the sample \(\xv \sim p_t\) follows the same distribution when transformed by \(\kappa \in \Gc\).

The following proposition generalize this result to all groups consisting of linear isometries and linear drifts of the form \(\uv(\xv, t) = u(t) \xv \).
\begin{restatable}{proposition}{SPDiffLO}
\label{cor:sp_diff_linear_op}
    Assume $\uv(\xv, t) = u(t)\xv$ for some scalar function $u: \Rb \rightarrow \Rb$. Given any group $\Gc$ (or $\Gv$) composed of linear isometries, if the unconditional $p_t$ induced by \cref{eq:fwd_diff_process} is $G$-invariant at $t = 0$, then it is $\Gc$-invariant for all $t \geq 0$. Likewise, if the conditional $q_t(\xv_t \vert \xv_T)$ induced by \cref{eq:ddbm_fwd} is $\Gv$-invariant at $t = 0$ then it is $\Gv$-invariant for all $t \geq 0$.
\end{restatable}

Since the drift terms of both VP and VE-SDE in \cref{tb:fg_selection} (\Cref{appx:diffusion_coefficients}) take the form \(u(t)\xv\), \cref{cor:sp_diff_linear_op} indicates that the induced diffusion process and the corresponding diffusion bridges are structure-preserving for any group composed of linear isometries.

\begin{figure*}
\begin{floatrow}
\CenterFloatBoxes
\ffigbox[0.19\columnwidth]{%
  \includegraphics[width=\columnwidth]{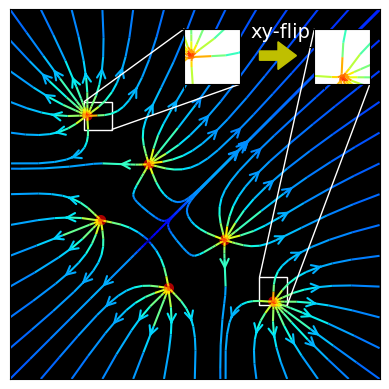}
}{%
  \caption{The vector fields of score functions that are equivariant under xy-flip.}
  \label{fig:eqv_score_func}
}
\hfill
\ffigbox[0.72\columnwidth]{%
	\vspace{0.8em}
    \centering
    \begin{minipage}{0.32\textwidth}
    \includegraphics[width=1 \textwidth]{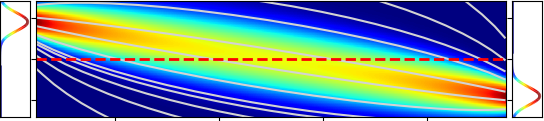}
    \includegraphics[width=1 \textwidth]{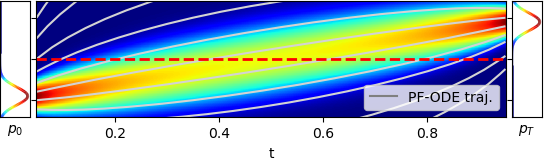}
    \end{minipage}
    \begin{minipage}{0.32\textwidth}
    \includegraphics[width=1 \textwidth]{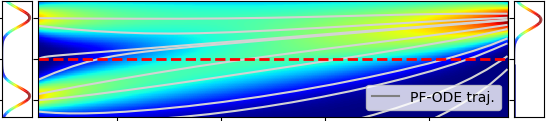}
    \includegraphics[width=1 \textwidth]{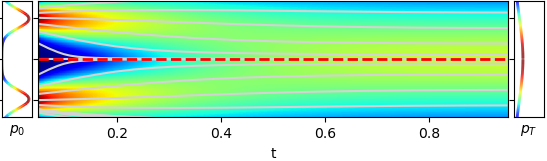}
    \end{minipage}
    \vspace{0.5em}
    \raisebox{-0.38\height}
    {~\includegraphics[width=0.032\textwidth]{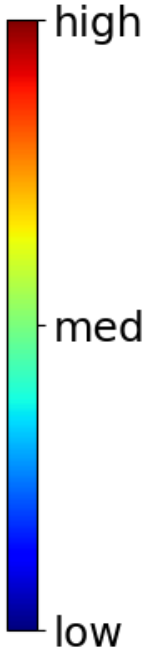}} 
}{%
  \caption{\textbf{Left:} The evolution of $p_t$ driven by DB processes induced by the VE-SDE in \cref{tb:fg_selection} (\Cref{appx:diffusion_coefficients})  conditioned on the end point $\xv_T = -1$ (upper) and $\xv_T = 1$ (lower). \textbf{Right:}  The upper plot has $f(x, t) = \frac{1-x}{1-t}$ and $g(t) = 1$. The lower is the VP-SDE in \cref{tb:fg_selection} (\Cref{appx:diffusion_coefficients})  with $\alpha_t = 1 -t$.}    
  \label{fig:eqv_sde_traj}
}
\vspace{-0.5em}
\end{floatrow}
\end{figure*}

\section{STRUCTURE PRESERVING MODELS}
\label{sec:model-architectures}



In this section, we explore applying the insights from \cref{sec:group_inv_sde} to ensure the data generated by SPDM adheres to a \(\Gv\)-invariant distribution. As mentioned in \cref{sec:preliminary}, sampling a diffusion model involves solving the SDE in \cref{eq:general_reverse_SDE_family} or \cref{eq:ddbm_bwd} by estimating the score using a neural network \(\sv_{\thetav}\). We will discuss several effective methods to design and train \(\sv_{\thetav}\) so that it meets the properties outlined in \cref{prop:eqv_drift}, achieving theoretically guaranteed structure-preserving sampling.  A summary of these methods limitations can be found in \cref{appx:limitations}.

\subsection{Structure Preserving Sampling}
\label{sec:inv_samp_dist}
\textbf{Unconditioned Distribution Sampling}. 
By \cref{prop:eqv_drift}, if a diffusion process is structure preserving, $p_t$ is $\Gc$-invariant for all $t\geq 0$. So given the prior distribution $p_T$ is $\Gc$-invariant and by \cref{lem:eqv_score}, for all $t \geq 0$, the score $\nabla_\xv \log p_t(\xv)$ is $\Gc$-equivariant. Thus, if the score estimator $\sv_{\thetav}(\xv, t)$ perfectly learns the $\Gc$-equivariant property and satisfies \cref{eq:eqv_drift_diff}, the drift of backward SDE~\cref{eq:general_reverse_SDE_family}:
\begin{align}
    \cev{\fv}_{\thetav, \lambda} (\xv_t, t) =  \uv(\xv_t, t) - \tfrac{1}{2}(1+\lambda^2) g^2(t) ~\sv_{\thetav}(\xv, t)\label{eq:bwd_drift}
\end{align}
also satisfies \cref{eq:eqv_drift_diff} as $[\zero]_{p_t}$ is closed under addition (see \cref{appx:distr_pres_drift}). Applying \cref{prop:eqv_drift} with reversed~$t$, we can then conclude that the generated samples must follow a $\Gc$-invariant distribution. 

\textbf{Equivariant Style-transfer Through Diffusion Bridges Conditioned on $\xv_T$}. When the drift $\uv(\xv, t)$ of original SDE in \cref{eq:fwd_diff_process} satisfies \cref{eq:eqv_drift_diff}, given stacked group $\Gv = \{(\kappa, \kappa)\vert \kappa \in \Gc\}$, we can show that the drift $\uv(\xv_t, t)+ g(t)^2 \hv(\xv_t,\xv_T, t)$
in \cref{eq:ddbm_fwd} also satisfies \cref{eq:eqv_drift_diff} (see \cref{prop:appx:equiv_bridge} in \cref{appx:SPDM} for the proof). As a result, by \cref{prop:eqv_drift}, $p_t(\xv_t\vert \xv_T)$ is $\Gv$-invariant for all $t \in [0,T]$ and thus by \cref{lem:eqv_score}, its score is equivariant and thus satisfies \cref{eq:eqv_drift_diff}. In this way, if the score estimator $\sv_{\thetav}(\xv_t, \xv_T, t)$ perfectly learns the equivariant property such that $\sv_{\thetav}(\kappa\xv_t, \kappa\xv_T, t) = \kappa\circ\sv_{\thetav}(\xv_t, \xv_T, t)$, the drift of reverse-time SDE \cref{eq:ddbm_bwd}
satisfies \cref{eq:eqv_drift_diff}; therefore, the invariant coupling between $\yv = \xv_T$ and $\xv_t$ is preserved for all $t \in [0,T]$ during the sampling process. 

Building on this observation, to ensure the generated data preserves the necessary geometric structure, it suffices to train a group equivariant score estimator $\sv_{\thetav}$. We present two theoretically guaranteed $\Gc$-equivariant implementations of $\sv_{\thetav}$, SPDM+WT and SPDM+FA.

\textbf{Weight Tying (SPDM+WT).}
Currently, most existing diffusion models are based on the U-Net backbone \citep{SalimansKCK2017, RonnebergerFB2015}. As the only components that are not equivariant are CNNs, we replace them with group-equivariant CNNs \citep{CohenW2016,RavanbakhshSP2017,EstevesABMD18,KondorTrivedi2018,KniggeRB2022,Yarotsky2022} to make the entire network equivariant. 

In particular, as we only consider linear isometry groups, we can make CNNs equivariant by tying the weights of the convolution kernels $\ks$, which will also reduce the total number of parameters and improve the computation efficiency \citep{RavanbakhshSP2017}. (For more general groups, refer to \cite{CohenW2016, KniggeRB2022} for methods to make CNNs \(\Gc\)-equivariant.) We provide more details on our selections of weight-tied kernels in \cref{appx:weight-tied-kernel}.

\textbf{Output Combining (SPDM+FA).}
When $\Gc$ contains finite elements, we can achieve $\Gc$-equivariance through frame averaging (FA) \citep{PunyASMGHL2022}, leveraging the following fact: for any function $\rv: \Rb^d \rightarrow \Rb^d$, 
\begin{align}
    \tilde\rv(\xv, \yv) = \tfrac{1}{\vert\Gc \vert}~{\textstyle\sum_{\kappa \in \Gc}} \kappa^{-1}~\rv(\kappa \xv, \kappa \yv)
\end{align}
is $\Gc$-equivariant, where $\vert\Gc \vert$ denotes the number of elements in $\Gc$ and the second argument of $\rv$ can be discarded for the approximation of the score not conditioned on $\yv$. Based on this fact, we can obtain an equivariant estimator $\tilde\sv_{\thetav}(\cdot, t)$ of the score by setting $\rv(\cdot) = \sv_{\thetav}(\cdot, t)$. Note that unlike other FA-based diffusion models~\citep{MartinkusLLVHRWCBGL2023,DuvalSHMMBR2023}, our method trains $\sv_{\thetav}(\cdot, t)$ using regular score-matching and only adopts FA during inference time. This design significantly saves training costs. To see why FA is not necessary during training, we note that our previous discussion has shown that the ground truth score function is equivariant; therefore, when $\sv_{\thetav}$ is well trained, estimator $\kappa^{-1}\sv_{\thetav} {(\kappa~\cdot, t)}$ for $\kappa \in \Gc$ will produce very similar output, estimating the value of the score function at $\xv_t$ (given $\yv$). Thus, their average is also a valid estimator. To boost the model's performance, additional regularizers can be used to encourage $\kappa^{-1}\sv_{\thetav} {(\kappa~\cdot, t)}$ for $\kappa \in \Gc$ to have the same output; we discuss this technique with extra details in \cref{appx:equivariance-regularization}.

\subsection{\texorpdfstring{$\Gc$}{Lg}-equivariant Trajectory}
\label{sec:eqv_samp_traj}

The sampling of SPDM is governed by the SDE in \cref{eq:general_reverse_SDE_family} or \cref{eq:ddbm_bwd}, collectively written as 
\begin{align}
    \diff \xv_t =  \cev{\fv}_{\thetav, \lambda}(\xv_t, \yv, t)\diff t + \lambda g(t) \diff \ws_t \label{eq:general_reverse_SDE_family:eqv}
\end{align}
where $\yv$ can be optionally discarded. In practice, the sampling process solves \cref{eq:general_reverse_SDE_family:eqv} iteratively through
\begin{align}
    \! \!\! \xv_{i-1}\!\leftarrow \! \cev{\fv}_{\thetav, \lambda}(\xv_i, \yv,t_i) (t_{i-1} \!-\! t_i)\! + \!\lambda g(t_i)\sqrt{t_i \!-\! t_{i-1}}{\epsilonv_i}\! \label{eq:euler_maruyama}
\end{align}
with preset time steps $\{t_i\}_{i=1}^n$ and $\epsilonv_i\sim\Nc(\zero, I)$. 

\begin{figure}[t!]
	\centering
    \includegraphics[width=0.6\textwidth]{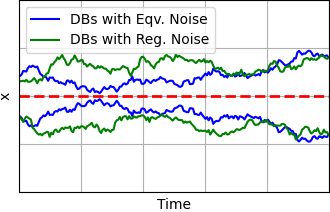}
    \caption{The trajectories of DBs with and without equivariant noise.}\label{fig:eqv_noise}
\end{figure}
While the techniques discussed in \cref{sec:inv_samp_dist} ensure the equivariance of \(\cev{\fv}_{\thetav, \lambda}\), they do not guarantee the equivariance of the sampled noise sequence \(\{\epsilonv_i\}_{i = 1}^n\), in a per-sample sense. As a result, the sample trajectory \(\xv_t\) may not be equivariant. This is visualized in \cref{fig:eqv_noise} with green curves, where the drifts of DBs are equivariant to the flip about $x = 0$ but the trajectory is not. Due to this asymmetry, the output of the SPDM will not be equivariant. One option to address this problem is to adopt ODE sampling by setting \(\lambda = 0\). However, this method is not always preferred as SDE sampling can significantly improve image quality \citep{KarrasAAL2022, SongDKKEP2021, ZhouLKE2024}. For \(\lambda > 0\), we also need \(\{\epsilonv_i\}_{i=1}^n\) to be equivariant in the sense that for \(\kappa \in \Gc\), if the starting point \(\xv_n\) is updated to \(\kappa \xv_n\), then \(\{\epsilonv_i\}_{i=1}^n\) is also updated to \(\{\tilde \epsilonv_i\}_{i=1}^n\) with \(\tilde \epsilonv_i = \kappa \epsilonv_i\). In this way, the trajectory becomes equivariant, as shown by the blue curves in \cref{fig:eqv_noise}. \cref{tab:eqv_noise} presents the effectiveness of sampling from an equivariant model with and without the noise satisfying the equivariant property. As shown, perfect equivariance is achieved only by combining equivariant noise (EN) and FA as seen in the SPDM+FA implementation.

In \cref{appx:eqv_aligner}, we present a simple technique to achieve EN by fixing the random seed and matching some artificial features between $\xv_n$ and $\epsilonv_n$. In our empirical study, we use this method to inject noise into $\xv_t$ as shown in \cref{fig:image-lysto-denosing} so that the rotation of the input results in a precise output rotation. 

\begin{table}[t]
    \centering
    \ttabbox{
    \caption{Equivariance of DB outputs when enforcing equivariant drift by frame-averaging (FA) and adding equivariant noise (EN).}\label{tab:eqv_noise}
    }{%
        \resizebox{0.9\columnwidth}{!}{%
        \begin{tabular}{@{}lcc@{}}
            \toprule
            & LYSTO & CT-PET  \\ \cmidrule(lr){2-3} 
            Model & $\Delta\hat\xv_0\downarrow$ & $\Delta\hat\xv_0\downarrow$ \\ \midrule
            DDBM  (Baseline)     & 52.44 & 164.91\\ 
            DDBM+FA              & 44.50 & 153.39 \\ 
            DDBM+EN              & 31.93 & 70.52  \\ 
            DDBM+FA+EN (SPDM+FA) & \textbf{0.00} & \textbf{0.00} \\ \bottomrule
        \end{tabular}%
    }
    }
\end{table}

\section{EMPIRICAL STUDY}
\label{sec:emp_study}

In this section, we present experiments demonstrating the effectiveness of the methods from \cref{sec:model-architectures}. The results support our theoretical work in \cref{sec:group_inv_sde} and offer additional insights. 

\paragraph{Datasets.} To evaluate the performance and equivariance capabilities of our models over image generation tasks, we adopt the rotated MNIST~\citep{LarochelleECBB2007}, LYSTO \citep{JiaoLYSTO2023}, and ANHIR \citep{BorovecEL2020} datasets. In order to validate equivariant trajectory sampling of our models we evaluate denoinsing LYSTO images, and style transfer from CT scan images to PET scan images of the same patient from the CT-PET dataset \citep{GatidisHFFNPSKCR2022}. See \cref{appx:datasets} for a more detailed discussion on the datasets used.

\paragraph{Models.} 
We implement regular diffusion models and bridge models (DDBM) \citep{ZhouLKE2024} based on VP-SDEs \citep{HoJA2020, SongME2021}, which are structure-preserving with respect to \(C_4\), \(D_4\), and flipping, as per \cref{cor:sp_diff_linear_op}. For generation tasks, we present the performance of the standard diffusion model, VP-SDE, as a baseline, along with SP-GAN \citep{BirrellKRLZ2022}, the only GAN-based model with theoretical group invariance guarantees. We also report the mean performance of GE-GAN \citep{DeyCG2021}. The tested models and their invariance and equivariance properties are summarized in \cref{tab:model-summary} (\Cref{appx:Hyperparameters}).

For style-transfer tasks, in addition to the original DDBM implementation \citep{ZhouLKE2024}, we report the performance of the popular style-transfer method Pix2Pix \citep{IsolaZZE2017} and the unconditional diffusion bridge model ${\rm I^2SB}$ \citep{LiuVHTNA2023}. For the denoising task on LYSTO, all models use pixel-space implementations. For the CT-PET dataset, all models except Pix2Pix are trained in a latent space, with images first encoded by a fine-tuned pretrained VAE from Stable Diffusion \citep{RombachBLEO2022}. FA was applied during fine-tuning and inference to ensure equivariance. 

All models, except SPDM-WT, are trained with data augmentation using randomly selected operators from their respective groups. We apply both non-leaky augmentation as in EDM \citep{KarrasAAL2022} and self-conditioning \citep{ChenCLLSZ2023} to improve diffusion model sample quality. For sampling we make use of both the DDPM and DDIM samplers \citep{HoJA2020,SongME2021}. Additional model configuration details can be found in \cref{appx:Hyperparameters} with training resources and training times in \cref{appx:model_resources}.

\begin{table}[t]
    \centering 
    \renewcommand{\arraystretch}{1.1} 
    \ttabbox{%
    \resizebox{0.97\columnwidth}{!}{%
    \begin{tabular}{lcccccc}
        \toprule
        \multicolumn{1}{c}{} & \multicolumn{6}{c}{Rotated MNIST} \\ \cline{2-7}
        \multirow{2}{*}{Model} & \multicolumn{4}{c}{FID$\downarrow$} & Inv-FID$\downarrow$ & $\Delta\hat\xv_0\downarrow$ \\
        \cmidrule(lr){2-5} \cmidrule(lr){6-7} & $1\%$ & $5\%$ & $10\%$ & $100\%$  & $100\%$ & $100\%$ \\ \hline
        VP-SDE & 5.97 & \textbf{3.05} & 3.47 & 2.81 & 2.21 & 36.98 \\
        SPDM+WT & 5.80 & 3.34 & 3.57 & 3.50 & 2.20 & \textbf{0.00} \\
        SPDM+FA & \textbf{5.42} & 3.09 & \textbf{2.83} & \textbf{2.64} & \textbf{2.07} & \textbf{0.00} \\
        SP-GAN & 149 & 99 & 88 & 81 & -- & -- \\
        SP-GAN (Reprod.) & 16.59 & 11.28 & 9.02 & 10.95 & 19.92 & -- \\ \hline
        GE-GAN & -- & -- & 4.25 & 2.90 & -- & -- \\
        GE-GAN (Reprod.) & 15.82 & 7.44 & 5.92 & 4.17 & 58.61 & -- \\ \hline
    \end{tabular}%
    }  
    }{
    \caption{Model Comparison on Rotated MNIST, LYSTO and ANHIR datasets.}\label{tab:fid-rot-mnist}
    }
    \vspace{0.5em}
    \resizebox{1\columnwidth}{!}{%
    \begin{tabular}{lcccccc}
        \toprule
        & \multicolumn{3}{c}{LYSTO} & \multicolumn{3}{c}{ANHIR} \\ \cmidrule(lr){2-4} \cmidrule(lr){5-7} 
        \multirow{2}{*}{Model} & \multirow{2}{*}{FID$\downarrow$} & \multirow{2}{*}{Inv-FID$\downarrow$} & \multirow{2}{*}{$\Delta\hat\xv_0\downarrow$} & \multirow{2}{*}{FID$\downarrow$} & \multirow{2}{*}{Inv-FID$\downarrow$} & \multirow{2}{*}{$\Delta\hat\xv_0\downarrow$} \\
        & & & & & & \\ \hline
        VP-SDE & 7.88 & 0.66 & 20.77 & 8.03 & 0.57 & 39.82 \\
        SPDM+WT & 12.75 & \textbf{0.59} & \textbf{0.00} & 11.73 & 0.43 & \textbf{0.00} \\
        SPDM+FA & \textbf{5.31} & 0.6 & \textbf{0.00} & \textbf{7.57} & \textbf{0.31} & \textbf{0.00} \\
        SP-GAN & 192 & -- & -- & 90 & -- & -- \\
        SP-GAN (Reprod.) & 16.29 & 0.66 & -- & 17.12 & 0.28 & -- \\ \hline
        GE-GAN & 3.90 & -- & -- & 5.19 & -- & -- \\
        GE-GAN (Reprod.) & 23.20 & 27.84 & -- & 14.16 & 6.87 & -- \\ \hline
    \end{tabular}
    }
\end{table}

\begin{table}[t]
    \centering
    \renewcommand{\arraystretch}{1.1} 
    \ttabbox{
    \resizebox{0.95\columnwidth}{!}{%
    \begin{tabular}{@{}lcccccccc@{}}
        \toprule
        & \multicolumn{4}{c}{LYSTO} & \multicolumn{4}{c}{CT-PET} \\ \cmidrule(lr){2-5} \cmidrule(lr){6-9}
        Model & FID$\downarrow$ & $L_1$ $\downarrow$ & SSIM$\uparrow$ & $\Delta\hat\xv_0\downarrow$ & FID$\downarrow$ & $L_1$ $\downarrow$ & SSIM$\uparrow$ & $\Delta\hat\xv_0\downarrow$ \\ \midrule
        DDBM & 17.28 & 0.076 & 0.696 & 52.44 & 18.13 & \textbf{0.041} & 0.861 & 164.91 \\
        SPDM+FA & \textbf{16.21} & \textbf{0.071} & 0.721 & \textbf{0.00} & \textbf{17.74} & 0.042 & 0.860 & \textbf{0.00} \\
        Pix2Pix & 78.43 & 0.087 & 0.654 & 113.63 & 20.26 & 0.043 & \textbf{0.862} & 172.11 \\
        ${\rm I^2SB}$ & 20.45 & 0.073 & \textbf{0.722} & 105.83 & 27.51 & 0.051 & 0.832 & 96.83 \\ \bottomrule
    \end{tabular}
    }
    }{
     \caption{Model Comparison on LYSTO denoising and CT-PET style transfer datasets.}\label{tab:bridge_models}
    }
\end{table}

\subsection{Image Generation Tasks}
\label{sec:results}
To demonstrate our models are able to match or exceed expected performance over the listed image generation datasets, we report the FID score \citep{HeuselRHTH2017} of each in \cref{tab:fid-rot-mnist}. To ensure consistency, we reproduced the results of SP-GAN and GE-GAN\footnote{We note that the reproduced FID of GE-GAN is significantly higher than reported by \citeauthor{DeyCG2021}, as their score is based on a customized {\rm InceptionV3} finetuned on LYSTO and ANHIR.  While included in the table for reference, these scores are not comparable with other FIDs.} and computed the FID using the standard {\rm InceptionV3} model. Details on our FID calculation are provided in \cref{appx:fid_computation}. Samples from each model are presented in \cref{appx:sample-images}. \cref{tab:fid-rot-mnist} shows that diffusion-based models consistently outperform SP-GAN across all datasets, with SPDM+WT and SPDM+FA performing on pair or better.

\begin{figure*}[t!]
    \centering
        \includegraphics[width=0.9\columnwidth]{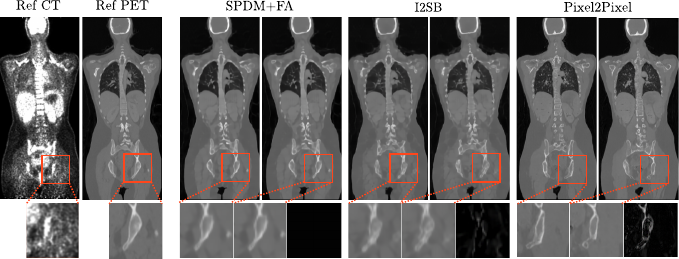}
        \caption{PET images generated by SPDM+FA, I2SB, and Pixel2Pixel for the CT-PET style transfer task. The leftmost group shows the input CT and ground-truth PET. }
    \label{fig:bridge_models_samples}
\end{figure*}

\textbf{Validating Equivariance.}
To quantify the degree of $\Gc$-invariance of the learned sampling distribution, we introduce a metric called \textit{Inv-FID}. Given a set of sampled images $\Dc_s$ from $\hat{p}_0$, Inv-FID calculates the maximum FID between $\kappa_1(\Dc_s)$ and $\kappa_2(\Dc_s)$ for $\kappa_1, \kappa_2 \in \Gc$. If $\hat{p}_0$ is perfectly \(\Gc\)-invariant, applying any $\kappa \in \Gc$ to its outcomes will leave the resulting distribution unchanged. Thus, the closer the FID between $\kappa_1(\Dc_s)$ and $\kappa_2(\Dc_s)$ is to zero, the more $\Gc$-invariant $\hat{p}_0$ is. As shown in \cref{tab:fid-rot-mnist}, diffusion models with theoretical guarantees, that is SPDM+WT and SPDM+FA, tend to achieve lower scores. Interestingly, the differences in Inv-FID scores across diffusion models are relatively small, indicating that these models naturally learn invariant properties. Therefore, in situations where invariance in the sampling distribution is not crucial, standard diffusion models may be sufficient.

\textbf{Validating Equivariant Sampling Trajectory.} 
To empirically validate our methods' theoretical guarantees of equivariant sampling trajectories, we implemented an image-denoising task using SDEdit \citep{MengHSSWZE2022}, as shown in \cref{fig:image-lysto-denosing}. Given a low-resolution (or corrupted) image \(\tilde\xv_0\), we add equivariant noise to obtain \(\xv_t\) using the technique from \cref{sec:eqv_samp_traj}. This technique is also applied when solving backward SDEs to obtain the denoised image \({\rm dn}(\tilde\xv_0)\), where \({\rm dn}\) represents the denoising process. As discussed in \cref{sec:eqv_samp_traj}, if a diffusion model is \(\Gc\)-equivariant, we should have \({\rm dn}(\kappa \, \tilde\xv_0) - \kappa \, {\rm dn} (\tilde\xv_0) \approx \zero\) for all \(\kappa \in \Gc\). In \cref{tab:fid-rot-mnist}, we report the average maximum pixel-wise distance \(\Delta \xv_0\) between \({\rm dn}(\kappa \, \tilde\xv_0)\) and \(\kappa \, {\rm dn} (\tilde\xv_0)\) over 16 randomly sampled corrupted \(\tilde\xv_0\) images. \(\kappa\) is randomly picked for each \(\tilde\xv_0\). The results show that theoretically equivariant models consistently have nearly zero \(\Delta \xv_0\), while models without theoretical guarantees produce significantly different outputs, which could be problematic in applications like medical image analysis. Likewise, in \cref{tab:bridge_models}, for a model $\mv_\theta$, we adopt a similar idea to measure its equivariance  
$\Delta \hat\xv_0$ by reporting the average maximum pixel-wise distance between 
\({\rm \mv_\theta}(\kappa \, \tilde\yv)\) and \(\kappa \, \mv_\theta (\tilde\yv)\), where $\yv$ is the input.

\textbf{Results.} Among models with theoretically guaranteed structure-preserving properties, SPDM+WT struggles to achieve FID scores comparable to FA methods on complex datasets like LYSTO. This is likely due to the weight-tying technique limiting the model's expressiveness and optimization. In contrast, SPDM+FA maintains sample quality and achieves the best performance on most datasets. This result corroborates our discussion made in \cref{sec:inv_samp_dist}, it being sufficient to train a score-based model using regular score-matching and combine the score-based model's outputs during the inference time to ensure equivariance without compromising the model's performance.

\subsection{Equivariant Image Style Transfer Tasks}
We compare the performance of Pix2Pix \citep{IsolaZZE2017}, ${\rm I^2SB}$ \citep{LiuVHTNA2023}, and our SPDM+FA models on the tasks of LYSTO image denosing and style-transfer converting a CT to PET scan image, as described at the start of the section. The results are shown in \cref{fig:bridge_models_samples}. In addition to FID for measuring sample quality and \(\Delta \hat\xv_0\) for measuring the model's equivariance, we report the \(L_1\) loss between the output and ground truth for local structure similarity and SSIM \citep{WangBSS2004} for global feature alignment. 

\textbf{Results.} SPDM+FA achieves nearly perfect group equivariance for both tasks, the best scores under most measures, and close-to-the-best scores in the rest. A qualitative sample comparison is provided in \cref{fig:bridge_models_samples}. Two images are generated from each model: the left conditioned on the original CT and the right an x-flipped CT that is x-flipped, reverted, again after generation. Perfect equivariance would result in two identical images. The bottom row zooms in on a selected area, showing the difference between the two samples. SPDM+FA's black patch indicates perfect equivariance, while the I2SB and Pixel2Pixel's, with white pixels, indicate imperfect equivariance. 

Beyond its perfect equivariance and the high image quality reflected in SPDM+FA’s low FID score (\cref{tab:bridge_models}), the visualizations in \cref{fig:bridge_models_samples} further show that our method generates PET images that closely match the ground truth, particularly in preserving bone shape. In contrast, other methods suffer from lower reconstruction accuracy and significant bone shape distortions when the input CT is flipped. Such issues could lead to misdiagnoses and unreliable clinical conclusions.	

These findings validate the effectiveness of the techniques introduced in \cref{sec:model-architectures} and reinforce our framework’s applicability to diffusion bridges, guiding the development of equivariant bridge models.

\section{DISCUSSION}

\label{sec:discussion}
In this paper, we investigated structure-preserving diffusion models (SPDM), an extended diffusion framework that accounts for invariants in the diffusion process. This extension allows us to effectively characterize the structure-preserving properties of a broader range of diffusion processes, including diffusion bridges used by DDBM \citep{ZhouLKE2024}. We presented a characterization of the drift terms that achieve a structure-preserving process, complementing existing work that primarily focuses on sufficient conditions. Based on the developed theoretical insights, we discussed several effective techniques to ensure the invariant distributions of samples and the equivariant properties of diffusion bridges. Empirical results on image generation and style-transfer tasks support our theoretical claims and demonstrate the effectiveness of the proposed methods in achieving structure-preserving sampling while maintaining high image quality.


\section*{Acknowledgements}

We would like to thank, Neel Dey, for providing the pre-processed ANHIR dataset used in \cite{DeyCG2021} and for clarifying some details around how the computation of FID was carried out within the forgoing paper. We also thank the reviewers and the area chair for the constructive comments. We gratefully acknowledge funding support from NSERC and the Canada CIFAR AI Chairs program. Resources used in preparing this research were provided, in part, by the Province of Ontario, the Government of Canada through CIFAR, and companies sponsoring the Vector Institute.
\FloatBarrier

\bibliography{invariance}
\bibliographystyle{apalike}

\section*{Checklist}
\label{sec:checklist}




\begin{enumerate}
 \item For all models and algorithms presented, check if you include:
 \begin{enumerate}
   \item A clear description of the mathematical setting, assumptions, algorithm, and/or model. [Yes]. All assumptions for the given theorem are clearly stated if needed, with more detailed derivations provided in the paper appendix. All dataset are described in detail, and model training paramters are provided in the appendix.
   \item An analysis of the properties and complexity (time, space, sample size) of any algorithm. [No] The primary motivation for the given experiments is in validating the theoretical guarantees from our primary proposition that characterizes diffusion model structure preserving for isometry groups. 
   \item (Optional) Anonymized source code, with specification of all dependencies, including external libraries. [No] We will provide anonymized source code if requested but not otherwise.
 \end{enumerate}

 \item For any theoretical claim, check if you include:
 \begin{enumerate}
   \item Statements of the full set of assumptions of all theoretical results. [Yes] Detailed proofs are provided in the appendix section.
   \item Complete proofs of all theoretical results. [Yes] Detailed proofs are provided in the appendix section.
   \item Clear explanations of any assumptions. [Yes] These are provided when necessary within the main body of the paper. All assumptions are stated explicitly in the appendix.     
 \end{enumerate}

 \item For all figures and tables that present empirical results, check if you include:
 \begin{enumerate}
   \item The code, data, and instructions needed to reproduce the main experimental results (either in the supplemental material or as a URL). [No] We will provide anonymized source code if requested but not otherwise. One of the medical imaging datasets used is under restricted licence due to privacy concerns, so we are not able to provide any data relating to this dataset such as processed data or model checkpoints.
   \item All the training details (e.g., data splits, hyperparameters, how they were chosen). [Yes] All key training parameters for the various models are provided in the appendix of the paper.
         \item A clear definition of the specific measure or statistics and error bars (e.g., with respect to the random seed after running experiments multiple times). [Yes] All custom measures are defined within the main text or a citation is provided for clarity. Additional details for some of the measures used in benchmarking are outlined in the paper Appendix.
         \item A description of the computing infrastructure used. (e.g., type of GPUs, internal cluster, or cloud provider). [Yes] We briefly mention the computing hardware used to train all the models mentioned in the paper in the introductory paragraph of the Empirical study section.
 \end{enumerate}

 \item If you are using existing assets (e.g., code, data, models) or curating/releasing new assets, check if you include:
 \begin{enumerate}
   \item Citations of the creator if your work uses existing assets. [Yes] We cite the work of any authors we make use of (e.g., code from existing machine learning models, dataset curators, etc.)
   \item The license information of the assets, if applicable. [Yes] One of the dataset used is under restricted licence due to primary concerns. This is mentioned in the main body of the text when the dataset is introduced.
   \item New assets either in the supplemental material or as a URL, if applicable. [Yes] All necessary information for replicating the dataset used in the experiments is either provided either in the main paper or in the supplemental sections.
   \item Information about consent from data providers/curators. [Yes] In so far as we mention necessary licences where required. 
   \item Discussion of sensible content if applicable, e.g., personally identifiable information or offensive content. [Yes] As mentioned above in the checklist, and in the main body of the paper, one of the dataset used is under restricted license due to patient privacy concerns. 
 \end{enumerate}

 \item If you used crowdsourcing or conducted research with human subjects, check if you include:
 \begin{enumerate}
   \item The full text of instructions given to participants and screenshots. [Not Applicable]
   \item Descriptions of potential participant risks, with links to Institutional Review Board (IRB) approvals if applicable. [Not Applicable]
   \item The estimated hourly wage paid to participants and the total amount spent on participant compensation. [Not Applicable]
 \end{enumerate}

\end{enumerate}

\onecolumn
\newpage
\appendix

\section{COMMON DIFFUSION PROCESS COEFFICIENTS}
\label{appx:diffusion_coefficients}

Here we provide a overview of some commonly used drift and diffusion coefficients, namely, those corresponding to he variance preserving~(VP, \citealt{HoJA2020, SongME2021}) and variance exploding~(VE, \citealt{SongDKKEP2021}) SDEs.

\begin{table}[!h]
    \centering
    \caption{Choices of $\uv(\xv, t)$ and $g(\xv)$ where $\eta_t=\frac{\alpha^2_t}{\sigma_t^2}$ \citep{ZhouLKE2024}.}
    \resizebox{1\linewidth}{!}{%
    \begin{tabular}{@{}cccccc@{}}
    \toprule
    SDE & $\uv(\xv, t)$                               & $g(t)^2$                                             & $p(\xv_t\vert \xv_0)$               & $\nabla_{\xv_t} \log p (\xv_T \vert \xv_t)$            &  $p_t(\xv_t \vert \xv_0, \xv_T)$                                                               \\ \midrule
    VP  & $ \frac{\diff \log \alpha_t}{\diff t}~\xv $ & $\frac{\diff\sigma_t^2}{\diff t}-\frac{\diff \log \alpha^2_t}{\diff t}\sigma_t^2$      & $\Nc(\alpha_t \xv_0, \sigma_t^2 \Iv)$ & $\frac{(\alpha_t / \alpha_T)\xv_T - \xv_t }{\sigma_t^2 (\eta_t / \eta_T - 1)}$ & $\Nc\Big(\frac{\eta_T}{\eta_t}\frac{\alpha_t}{\alpha_T} \xv_T + \alpha_t \xv_0(1-\frac{\eta_T}{\eta_t}), \sigma_t^2 (1 -\frac{\eta_T}{\eta_t})\Big)$
    \\ \midrule
    VE  & $\zero$                                     & $\frac{\diff \sigma_t^2}{d t}$ & $\Nc(\xv_0, \sigma_t^2 \Iv)$          & $\frac{\xv_T - \xv_t}{\sigma_T^2 - \sigma_t^2}$         & $\Nc\Big(\frac{\sigma_t^2}{\sigma_T^2} \xv_T + (1- \frac{\sigma_t^2}{\sigma_T^2})\xv_0 , \sigma_t^2 (1- \frac{\sigma^2_t}{\sigma^2_T})\Big)$                                            \\ \bottomrule
    \end{tabular}
        }
    \label{tb:fg_selection}
\end{table}

\section{DERIVATION DETAILS OF THE THEORETICAL RESULTS}
\label{appx:proofs}


In this section, we provide detailed derivations of our theoretical results. For conciseness, we complete most of the proofs in measure theory notation and show their equivalence to those presented in the main text.

In \cref{appx:isometry}, we demonstrate that the isometry assumption on group operators results in their linearity. In \cref{appx:liouville}, we briefly review the Liouville equations, which play a crucial role in characterizing the distribution evolution of particles driven by an ODE drift. In \cref{appx:distr_pres_drift}, we discuss a special family of ODE drifts that preserve distributions, characterizing the equivalence of various drifts by inducing the same evolution of \(p_t\). This result helps derive the equivalent conditions on drifts to achieve structure-preserving ODE and SDE processes. \cref{appx:SP_ODE} discusses the structure-preserving conditions for ODE processes, and we extend the results to SDE processes in \cref{appx:SPDM}.

\subsection{Isometries}
\label{appx:isometry}
The groups $\Gc$ involved in our discussions from \cref{sec:group-invariance} are assumed to consist of isometries $\kappa$ satisfying $\|\kappa \xv\| = \|\xv\|$. Then, for $\kappa \in \Gc$ and $\xv, \yv \in \Rb^d$, we have $\| \kappa \xv - \kappa \yv \|_2 = \|\xv- \yv\|_2$. Since $\kappa \in \Gc$ is bijective, by the Mazur–Ulam theorem \citep{Nica2012}, $\kappa$ is affine and thus can be written as 
\begin{align}
    \kappa(\xv) = A \xv + \bv
\end{align}
for some $A_\kappa \in \Rb^{d \times d}$ and $\bv_\kappa \in \Rb^d$. Besides, $A_\kappa$ is orthogonal:
\begin{lemma}
\label{lem:iso_jacob_orth}
    If $\kappa(\xv) = A_\kappa \xv + \bv_\kappa$ is an isometry, then $A_\kappa^\top A_\kappa = \Iv$.
\end{lemma}
\begin{proof}
    As $\kappa$ is an isometry, then for $\xv, \yv \in \Rb^d$,
    \begin{align}
        \|A_\kappa\xv -A_\kappa\yv\| = \|\kappa \xv - \kappa \yv\| = \|\xv -\yv\|.
    \end{align}
    In addition,
    \begin{align}
        \inner{A_\kappa\xv}{A_\kappa\yv} = \frac{1}{4} \left[ \|A_\kappa\xv - A_\kappa(-\yv) \|^2 - \| A_\kappa\xv - A_\kappa\yv \|^2 \right] = \inner{\xv}{\yv}
    \end{align}
    That is $\inner{A_\kappa^\top A_\kappa \xv}{\yv} = \inner{\xv}{\yv}$, which implies $A_\kappa^\top A_\kappa = \Iv$. 
\end{proof}
\begin{remark}
\label{rmk:isometry_jacob}
    \Cref{lem:iso_jacob_orth} suggests that $D \kappa (\xv) = A_\kappa$ for all $\xv \in \Rb^d$. 
\end{remark}
In addition, since $\|\kappa \xv\|_2 = \|\xv\|_2$, we have $\|A_\kappa \xv + \bv\| = \|\xv\|$ for all $\xv$. Setting $A_\kappa \xv = -\bv$ yields $\|\bv\| = 0$, or equivalently, $\bv = \zero$. 

\fbox{\begin{minipage}{\textwidth}
Therefore, for all the group operators $\kappa$ appearing in our discussion, we can write:
\begin{align}
    \kappa \xv = A_\kappa \xv
\end{align}
for some orthogonal $A_\kappa \in \Rb^{d \times d}$. 
\end{minipage}}

The following lemma can significantly simplify the discussion on the geometric properties of diffusion processes in \cref{appx:liouville} and \cref{appx:SPDM}:
\begin{lemma}
    \label{lem:comp_supp_inv_under_kappa}
    Let $C_c^\infty(\Rb^d)$ be the set of compactly supported functions. Then for $\kappa\in \Gc$,
    \begin{align}
        \{\phi \circ \kappa \vert \phi \in C_c^\infty(\Rb^d)\} = C_c^\infty(\Rb^d)
    \end{align}
\end{lemma}
\begin{proof}
    If $\phi \in C_c^\infty(\Rb^d)$ has a compact support $C$, then $\phi \circ \kappa$ has a support $\kappa^{-1} C$, which is also compact because $\kappa^{-1} \in \Gc$ is also affine (thus continuous) and a continuous image of a compact set is compact. Moreover, since $\phi$ is infinitely differentiable, so is $\phi \circ \kappa$. Thus, $\{\phi \circ \kappa \vert \phi \in C_c^\infty(\Rb^d)\} \subseteq C_c^\infty(\Rb^d)$. In addition, for $\psi \in C_c^\infty(\Rb^d)$, we have $\phi = \psi \circ \kappa^{-1} \in C_c^\infty(\Rb^d)$ such that $\phi \circ \kappa = \psi$. Hence, $C_c^\infty(\Rb^d) \subseteq \{\phi \circ \kappa \vert \phi \in C_c^\infty(\Rb^d)\}$.
\end{proof}

\subsection{Liouville Equation}
\label{appx:liouville}

Our proof relies on the Liouville equation in measure theory notation. We provide an intuitive and easy-to-follow proof here and show its equivalence to the popular version in the probability density notation in \Cref{rmk:eqv_liouville_density}. 

Consider $N$ non-interacting particles moving according to a deterministic ODE in $\Rb^d$:
\begin{align}
	\diff \xv_t = \uv(\xv_t, t) \diff t.  \label{eq:appx:ode_process}
\end{align}
Then their distribution is characterized by a measure $\mu^{(N)}_t$ such that for any compactly supported function $\phi \in C_c^\infty(\Rb^d)$ we have
\begin{align}
	\int \phi(\xv) \diff \mu^{N}_t(\xv) = \frac{1}{N}\sum_{i = 1}^{N} \phi(\xv^i_t). 
\end{align}
Then 
\begin{align}
	\frac{\partial}{\partial t}\int \phi(\xv) \diff \mu^{N}_t(\xv) 
	&= \frac{1}{N} \frac{\diff}{\diff t}\sum_{i=1}^N \phi(\xv^i_t)  
	= \frac{1}{N} \sum_{i=1}^N \nabla \phi(\xv^i_t) \cdot \uv(\xv_t^i, t)\\
	 &= \int \nabla \phi(\xv) \cdot \uv(\xv, t) \diff \mu^{N}_t(\xv). 
\end{align}
Then if we suppose the initial distribution 
\begin{align}
	\mu^{N}_0 (\xv, 0) \rightarrow^\star \mu_0(\xv)~\mbox{as}~N\rightarrow \infty
\end{align}
in a sense that $\int\phi(\xv) \diff \mu_0(\xv) \rightarrow \int\phi(\xv) \diff \mu_0(\xv)$ for any $\phi \in C_c^\infty(\Rb^d)$. Then we can establish the limit $\mu_t^N(\xv) \rightarrow^\star \mu_t(\xv)$ and $\mu_t$ satisfies 
\begin{align}
\label{eq:liouville_meas}
	\frac{\partial}{\partial t}\int \phi(\xv) \diff \mu_t(\xv) = \int \nabla \phi(\xv) \cdot \uv(\xv, t) \diff \mu_t(\xv). 
\end{align}

Notably, the setting we consider in the main text assume that the drift could optionally depend on some additional (fixed) term $\yv$ such that
\begin{align}
    \diff \xv_t = \fv(\xv_t, \yv, t) \diff t,
\end{align}
where $\xv_t \in \Rb^m$ and $\yv \in \Rb^n$ with $m > 0$ and $n \geq 0$.\footnote{We use $n = 0$ to indicate the case that $\fv$ does not depend on $\yv$. Unless otherwise stated, we will continue to use this convention.}
In this case, the process can be rewritten as
\begin{align}
    \diff \begin{bmatrix}
        \xv_t \\ 
        \yv
    \end{bmatrix} = \begin{bmatrix}
        \fv(\xv_t, \yv, t)\\
        \zero
    \end{bmatrix} \diff t = \uv\big([\xv_t, \yv]^\top, t \big) \diff t.
\end{align}
Applying \cref{eq:liouville_meas}, we obtain
\begin{equation}
\label{eq:liouville_meas_cond}
\boxed{
\frac{\partial}{\partial t} \int \phi(\xv, \yv) \diff \mu_t(\xv \vert \yv) =  \int \hspace{-0.2em}\nabla_1  \phi(\xv, \yv) \cdot \fv(\xv_t, \yv, t) \diff \mu_t(\xv \vert \yv) \qquad \emph{(Liouville eq. meas.)}
}
\end{equation}
where $\nabla_1 \psi (\xv, \yv, \ldots) := \frac{\partial \psi (\xv, \yv, \ldots)}{\partial \xv}$ denote the gradient with respect to the first argument.

\begin{remark}
\label{rmk:eqv_liouville_density}
Let $\lambda$ denote the Lebesgue measure. When the probability measure $\mu_t(\xv \vert \yv)$ has density $p_t(\xv \vert \yv) \in C^1(\Rb^m \times \Rb^n \times [0,T])$ with respect to $\xv$, we have 
\begin{align*}
&~~~~~~\frac{\partial}{\partial t}\int_c \phi(\xv, \yv) p_t(\xv \vert \yv) \diff \lambda(\xv) = \frac{\partial}{\partial t}\int_c \phi(\xv, \yv) \diff \mu_t(\xv \vert \yv)\\
	&= \int_c \nabla_1 \phi(\xv, \yv) \cdot \fv(\xv, \yv, t) \diff \mu_t(\xv\vert \yv) = \int_c \nabla_1 \phi(\xv, \yv) \cdot \fv(\xv, \yv, t) p_t(\xv
 \vert \yv) \diff \lambda(\xv) \\
	&= [ p_t(\xv \vert \yv)\fv(\xv, \yv, t)\cdot\phi(\xv, \yv)]_{\partial  c} - \int_c \phi(\xv, \yv) ~ \nabla_\xv \cdot\big(p_t(\xv\vert \yv) \fv(\xv, \yv, t)\big) \diff \lambda(\xv) \\
	&=  - \int_c \phi(\xv, \yv)~ \nabla_\xv \cdot \big(p_t(\xv\vert \yv) \fv(\xv, \yv, t) \big) \diff \lambda(\xv). 
\end{align*}
As this holds for all $\phi \in C_c^\infty(\Rb^{m+n})$, we obtain the regular Liouville equation \citep{Oksendal2003, Ehrendorfer2006}:
\begin{equation}
\label{eq:liouville_density}
\boxed{
	\frac{\partial}{\partial t}p_t(\xv \vert \yv) = - \nabla_\xv \cdot\big(p_t(\xv\vert \yv) \fv(\xv, \yv, t) \big). \qquad\emph{(Liouville eq. density)}
 }
\end{equation}
\vspace{1em}
\end{remark}

\subsection{Distribution-preserving Drifts.}
\label{appx:distr_pres_drift}
While a zero drift implies $p_t = p_0$ for all $t > 0$, the converse is not necessarily true:
\begin{example}
\label{ex:nonzero_drift_preserve_dist}
    Let $p_0$ be the density of a spherical Gaussian $\Nc(\zero, \Iv)$ in $\Rb^2$. For $\fv(\xv, t) = [y, -x]^{\top}$, by the Liouville equation, at $t = 0$
    \begin{align}
        \frac{\partial}{\partial t} p_t(x , y) &= - \nabla \cdot \Big[\frac{1}{Z} \exp(-\frac{x^2 + y^2}{2}) \,[y, -x]^\top\Big]\\
        &= \frac{\partial}{\partial x}\left[\frac{1}{Z} \exp(-\frac{x^2 + y^2}{2}) y \right] - \frac{\partial}{\partial y}\left[\frac{1}{Z} \exp(-\frac{x^2 + y^2}{2}) x \right] \\
        &= \frac{1}{Z}\left[-\exp(-\frac{x^2 + y^2}{2})xy + \exp(-\frac{x^2 + y^2}{2})xy \right] = 0.
    \end{align}
    As a result, $\fv$ does not change $p_0$, although it is not zero. 
\end{example}
In general, 
\begin{lemma}
   Given a measure $\mu$, drift $\fv$ does not change the distribution if for all $\phi \in  C_c^\infty(\Rb^{d})$
\begin{align}
  0 = \int \nabla \phi(\xv) \cdot \fv(\xv, t) \diff \mu(\xv)  \label{eq:def_dist_prev_drift}
\end{align} 
for all $\xv$ and $t$. We use $[\zero]_{\mu}(\xv)$ to denote the set of drifts that do not alter distribution $\mu$. That is, if $\fv$ satisfies \cref{eq:def_dist_prev_drift}, we have $\fv \in [\zero]_{\mu}$.
\end{lemma}
\begin{proof}
    This is an immedate result of \cref{eq:liouville_meas_cond} by setting the left-hand side zero. 
\end{proof}
\begin{remark}
For any $\mu$, $\zero \in [\zero]_{\mu}$.  
\end{remark}
\begin{remark}
If $\fv, \gv \in [\zero]_{\mu}$, then $\alpha \fv + \beta \gv \in [\zero]_{\mu}$, for $\alpha, \beta \in \Rb$, \ie, $[\zero]_{\mu}$ is a vector space. 
\end{remark}
\begin{remark}
    In the main text, we use the notation \([\zero]_p\) instead of \([\zero]_\mu\) to represent distribution-preserving drifts that maintain a distribution with density \(p\), which corresponds to the distribution measure \(\mu\).
\end{remark}

\subsection{Structural Preserving ODE Processes}
\label{appx:SP_ODE}
In this section, we discuss the sufficient and necessary condition of structurally preserved ODE processes. 
Here, we consider ODE process:
\begin{align}
    \diff \xv_t = \fv(\xv_t, \yv, t) \diff t \label{eq:ode_process_general}
\end{align}
with $\xv_t \in \Rb^m$, and $\yv \in \Rb^n$ denote additional conditions of the process. Here, we assume $m > 0$ and $n \geq 0$, where $n = 0$ denote the case when $\fv$ does not depend on $\yv$. We note that for a similar setting with discrete time step and drift $\fv$ that does not depend on $\yv$, a sufficient condition on $\Gc$-invariance of $\mu_t$ for $t \geq 0$ has been discussed by \cite{PapamakariosNRDML2021} and \cite{KohlerKN2020}.  

Let $\mu_t(\xv_t \vert \yv)$ be the probability measure of $\xv_t$ induced by the ODE process \eqref{eq:ode_process_general} conditioned on $\yv$. Let $\Gv = \{\kappav = (\kappa_1, \kappa_2) | \kappa_1: \Rb^m \rightarrow \Rb^m,  \kappa_2: \Rb^n \rightarrow \Rb^n\}$ be a group of isometries defined in $\Rb^{m+n}$ such that $\kappav (\xv, \yv) = (\kappa_1 \xv, \kappa_2 \yv)$. It is easy to see that the sets of $\kappa_1$ and $\kappa_2$ are also groups of isometries. We will respectively denote them as $\Gc_1$ and $\Gc_2$. In addition, by \cref{lem:iso_jacob_orth}, we have
\begin{align}
    \kappav(\xv, \yv) = 
    \begin{bmatrix}
        A_{\kappa_1} & \zero\\
        \zero & A_{\kappa_2}
    \end{bmatrix}
    \begin{bmatrix}
        \xv \\
        \yv
    \end{bmatrix}
    +
    \begin{bmatrix}
        \bv_{\kappa_1}\\
        \bv_{\kappa_2}
    \end{bmatrix},
\end{align}
where $A_{\kappa_1} \in \Rb^{m\times m}$ and $A_{\kappa_2} \in \Rb^{n\times n}$ are orthogonal. 

Furthermore, by \cref{rmk:isometry_jacob}, we have
\begin{align}
    D\kappav = A_{\kappav} = 
    \begin{bmatrix}
        A_{\kappa_1} & \zero\\
        \zero & A_{\kappa_2}
    \end{bmatrix}.
\end{align}

We say $\mu_t(\xv_t \vert \yv)$ is \emph{$\Gv$-invariant} if for all $\kappav \in \Gv$, $\mu_t(\kappa_1\xv_t|\kappa_2 \yv) = \mu_t(\xv_t| \yv)$. \cref{lem:inv_prob_density} shows that this definition is equivalent to the one given in \cref{sec:group-invariance}.

\begin{lemma}
\label{lem:inv_prob_density}
Assume $\mu(\cdot \vert \cdot )$ has density $p(\xv \vert \yv)$. Then
$\mu(\cdot \vert \cdot )$ is $\Gv$-invariant if and only if the density $p(\xv \vert \yv) = p(\kappa_1 \xv \vert \kappa_2 \yv)$ for all $\kappav \in \Gv$, $\xv \in \Rb^m$ and $\yv \in \Rb^n$. 
\end{lemma}
\begin{proof}
$\mu(\cdot \vert \cdot )$ is $\Gv$-invariant if and only if for all $\kappav \in \Gv$, $\phi \in C_c^\infty(\Rb^{m+n})$, 
\begin{align}
	\int \phi(\xv, \yv) \diff \muv(\xv \vert \yv) = \int \phi(\kappav^{-1} (\xv, \yv)) \diff \muv(\xv\vert \yv). 	
\end{align}
That is, 
\begin{align*}
	&\int \phi(\xv, \yv)p(\xv\vert \yv)\diff \lambda(\xv) = \int \phi(\kappav^{-1} (\xv, \yv)) \diff \muv(\xv\vert \yv) = \int \phi(\xv, \yv) \diff \mu(\kappa_1 \xv \vert \kappa_2 \yv) \\
	=& \int \phi(\xv, \yv) p(\kappa_1 \xv\vert \kappa_2 \yv)\diff \lambda(\kappa_1 \xv) \overset{(\rm \cref{lem:inv_leb})}{=}  \int \phi(\xv, \yv) p(\kappa_1 \xv\vert \kappa_2 \yv)\diff \lambda(\xv).
\end{align*}
Therefore, $p(\xv\vert \yv) = p(\kappa_1\xv\vert \kappa_2 \yv)$. Since every step is reversible, the proof is completed.
\end{proof}

Then we give the equivalent conditions on the drift terms to ensure the structure-preserving property of ODE flows. 
\begin{lemma}
\label{lem:eqv_drift}
Consider the ODE process in \eqref{eq:ode_process_general} with $\Gv$-invariant  $\mu_0(\cdot \vert \cdot)$. Then, $\mu_t$ is $\Gv$-invariant for all $t\geq 0$  if and only if
\begin{align}
    A_{\kappa_1}^\top \fv(\kappa_1 \xv, \kappa_2 \yv, t) - \fv(\xv, \yv, t) \in [\zero]_{\mu_t}.
   \label{eq:eqv_drift_ode}
\end{align}
for all $t\geq 0$, $\xv \in \Rb^m$, $\yv \in \Rb^n$ and $\kappav = (\kappa_1, \kappa_2) \in \Gv$.
\end{lemma}
\begin{proof}
($\Rightarrow$) Assume that $\mu_t$ is $\Gv$-invariant for all $t\geq 0$. 
For all $\phi \in C_c^\infty(\Rb^{m + n})$ and $\kappav \in \Gv$, let $\psi = \phi \circ \kappav$. We note that by \cref{lem:comp_supp_inv_under_kappa}, $\psi \in  C_c^\infty(\Rb^{m + n})$. Then, for $t \geq 0$, we have
\begin{align*}
	0 &= \frac{\diff}{\diff t} \int \phi(\xv, \yv) \diff \mu_t(\xv|\yv) - \frac{\diff}{\diff t} \int \phi(\xv, \yv) \diff \mu_t(\kappa_1^{-1}\xv\vert \kappa_2^{-1}\yv) \\
	&= \frac{\diff}{\diff t} \int \phi(\xv, \yv) \diff \mu_t(\xv|\yv) - \frac{\diff}{\diff t} \int \phi(\kappa_1\xv, \kappa_2 \yv) \diff \mu_t(\xv\vert \yv) \\
	\overset{\eqref{eq:liouville_meas_cond}}&{=} \int \nabla_1\phi(\xv, \yv)^\top \fv(\xv, \yv, t) \diff \mu_t(\xv \vert \yv) - \int \nabla_1\psi(\xv, \yv)^\top \fv(\xv, \yv, t) \diff \mu_t(\xv\vert \yv) \\
	\overset{(\rm \Gv-inv)}&{=} \int \nabla_1\phi(\kappa_1 \xv, \kappa_2 \yv)^\top \fv(\kappa_1 \xv,\kappa_2\yv, t) \diff \mu_t(\xv\vert \yv) - \int \nabla_1\psi(\xv, \yv)^\top \fv(\xv, \yv, t) \diff \mu_t(\xv\vert \yv) \\
&= \int \nabla_1\psi(\xv, \yv)^\top D \kappa_1(\xv)^\top \fv(\kappa_1 \xv,\kappa_2\yv, t) \diff \mu_t(\xv\vert \yv) - \int \nabla_1\psi(\xv, \yv)^\top \fv(\xv, \yv, t) \diff \mu_t(\xv\vert \yv) \\
&= \int \nabla_1\psi(\xv, \yv)^\top \Big(A_{\kappa_1}(\xv)^\top \fv(\kappa_1 \xv, \kappa_2  \yv, t) -  \fv(\xv, \yv, t) \Big) \diff \mu_t(\xv\vert \yv).
\end{align*}
By \cref{lem:comp_supp_inv_under_kappa}, $\psi$ can be any functions in $C_c^\infty(\Rb^{m + n})$. Thus, \cref{eq:eqv_drift_ode} follows.

\vspace{1em}

\noindent
($\Leftarrow$)
Assume \eqref{eq:eqv_drift_ode} holds and $\muv_0$ is $\Gv$-invariant. For all $\phi \in C_c^\infty(\Rb^{m + n})$ and $\kappav \in \Gv$, let $\psi = \phi \circ \kappav$. 
 Then we have
\begin{align*}
	&~~~~ \frac{\diff}{\diff t}\int \phi(\xv, \yv) \diff\muv_t(\kappa_1^{-1} \xv\vert  \kappa_2^{-1}\yv) = \frac{\diff}{\diff t} \int \phi(\kappa_1 \xv, \kappa_2 \yv) \diff \muv_t(\xv\vert \yv) = \frac{\diff}{\diff t} \int \psi(\xv, \yv) \diff \muv_t(\xv\vert \yv) \\
	\overset{\eqref{eq:liouville_meas_cond}}&{=} \int (\nabla_1\psi)(\xv, \yv)^\top \fv(\xv, \yv, t) \diff \mu_t(\xv \vert \yv)
 \overset{(\Av)}{=} \int (\nabla_1\phi)\,(\kappa_1 \xv, \kappa_2 \yv)^\top \fv(\kappa_1\xv, \kappa_2\yv, t) \diff \muv_t(\xv\vert \yv)\\
	&= \int (\nabla_1\phi)\,(\xv, \yv)^\top ~ \fv(\xv, \yv, t) \diff \muv_t(\kappa^{-1}_1\xv\vert \kappa^{-1}_2\yv), \numberthis\label{eq:eqv_drift_ode:2}
\end{align*}
where (\textbf{A}) is due to:
\begin{align*}
    0 \overset{\eqref{eq:eqv_drift_ode}}&{=} \int \nabla_1 \psi (\kappa_1 \xv, \kappa_2 \yv)^\top \Big(A_{\kappa_1}(\xv)^\top \fv(\kappa_1 \xv, \kappa_2  \yv, t) -  \fv(\xv, \yv, t) \Big) \diff \mu_t(\xv\vert \yv) \\
    &= \int \nabla_1\psi(\kappa_1 \xv, \kappa_2 \yv)^\top D \kappa_1(\xv)^\top \fv(\kappa_1 \xv,\kappa_2\yv, t) \diff \mu_t(\xv\vert \yv) - \int \nabla_1\psi(\xv, \yv)^\top \fv(\xv, \yv, t) \diff \mu_t(\xv\vert \yv) \\
    &= \int \nabla_1\phi(\kappa_1 \xv, \kappa_2 \yv)^\top \fv(\kappa_1 \xv,\kappa_2\yv, t) \diff \mu_t(\xv\vert \yv) - \int \nabla_1\psi(\xv, \yv)^\top \fv(\xv, \yv, t) \diff \mu_t(\xv\vert \yv) 
\end{align*}
Besides, we have
\begin{align}
	\frac{\diff}{\diff t}\int \phi_1( \xv, \yv) \diff\mu_t(\xv \vert \yv) \overset{\eqref{eq:liouville_meas_cond}}{=} \int \nabla\phi_1(\xv, \yv)^\top \fv(\xv, \yv, t) \diff \mu_t(\xv \vert \yv). \label{eq:eqv_drift_ode:3}
\end{align}
As $\mu_0(\xv \vert \yv)=\mu_0(\kappa_1^{-1}\xv \vert \kappa_2^{-1}\xv)$, \eqref{eq:eqv_drift_ode:2} and \eqref{eq:eqv_drift_ode:3} together suggest that $\mu_t(\xv\vert \yv)$ and $\mu_t(\kappa_1^{-1}\xv\vert \kappa^{-1}_2\yv)$ share the same Liouville's equation. Therefore, $\mu_t(\xv \vert \yv) = \mu_t(\kappa_1^{-1}\xv\vert \kappa_2^{-1} \yv)$ for all $t\geq 0$. 
\end{proof}

\subsection{Structural Preserving SDE Processes}
\label{appx:SPDM}
In this section, we assume all the measures involved have densities. We first show that Lebesgue measure is $\Gc$-invariant, where $\Gc$ is a group of isometries.
\begin{lemma}
	For all $\kappa\in \Gc$, $\det D\kappa (\xv) = \det A_\kappa = 1$ or $-1$ for all $\xv \in \Rb^d$. 
\end{lemma}
\begin{proof}
For $\kappa\in \Gc$, by \cref{lem:iso_jacob_orth} and \cref{rmk:isometry_jacob}, we have $D\kappa (\xv)^\top D\kappa (\xv) = A_\kappa^\top A_\kappa = I$. Then $\det(D\kappa (\xv))^2 = (\det A_\kappa)^2 = 1$, which implies $\det D\kappa (\xv) = \det A_\kappa = \pm 1$
\end{proof}
\begin{lemma}
\label{lem:inv_leb}
The Lebesgue measure $\lambda$ is $\Gc$-invariant.
\end{lemma}
\begin{proof}
  For all $\phi\in C_c^\infty(\Rb^{d})$ and $\kappa \in \Gc$, we have
  \begin{align*}
    \int\phi(\xv)\diff \lambda(\xv) &= \int \phi(\kappa \xv) \diff \lambda(\kappa \xv) \overset{\rm (\cref{lem:inv_leb})}{=} \int\phi(\kappa \xv) \, \vert \det D \kappa(\xv) \vert \diff \lambda(\xv)\\
     &= \int\phi(\kappa \xv)  \diff \lambda(\xv) = \int \phi(\xv) \diff \lambda(\kappa^{-1} \xv)
  \end{align*}
Therefore, $\lambda = \kappa^{\#}\lambda$.
\end{proof}

To deal with the invariance property associated with the diffusion term, we prove a lemma similar to Lem F.4 of \citep{YimTDMDBJ2023}. The lemma basically says Laplacian is invariant with respect to isometries:
\begin{lemma}
\label{lem:inv_laplacian}
	For $\kappav \in \Gv$ and $v: \Rb^{m+n} \rightarrow \Rb$, we have
	\begin{align}
		\Delta_1 (v \circ \kappav) (\xv, \yv) = (\Delta_1 v) \circ \kappav(\xv, \yv),
	\end{align}
	where
	\begin{align}
		&(\Delta_1 u) \, (\xv, \yv) = \sum_{k=1}^m \frac{\partial^2}{\partial x_k^2} u(\xv, \yv) = (\nabla_1 \cdot \nabla_1 u) (\xv, \yv),\\
		& (\nabla_1 u) \, (\xv, \yv) = \frac{\partial u}{\partial \xv} (\xv,  \yv). 
	\end{align}
\end{lemma}
\begin{proof}
 Let
	\begin{align}
	M = \begin{bmatrix}
			\Iv_m & \zero\\
			\zero & \zero
		\end{bmatrix}
	\end{align}
	where $\zero$ denotes a zero matrix of a proper size. Then it is easy to see
	\begin{align}
		\Delta_1 v(\xv, \yv) = \nabla \cdot \Big( M~\nabla v  (\xv, \yv) \Big). 
	\end{align}
	As a result, for all $\phi\in C_c^\infty(\Rb^{m + n})$, we have	
	\begin{align*}
		 &\int \phi(\xv, \yv) (\Delta_1 v) \circ \kappav(\xv, \yv) \diff \lambda(\xv, \yv) = \int \phi(\kappa_1^{-1}\xv, \kappa_2^{-1} \yv) ~ (\Delta_1 v)(\xv, \yv) \diff \lambda(\xv, \yv)\\
    &= \int\phi(\kappa_1^{-1}\xv, \kappa_2^{-1} \yv) ~ \nabla \cdot \Big( M~\nabla v  (\xv, \yv) \Big) \diff \lambda(\xv, \yv)\\
    &= \left[\phi(\kappa_1^{-1}\xv, \kappa_2^{-1} \yv)~M \nabla v(\kappa_1^{-1}\xv, \kappa_2^{-1} \yv)\right]_{\partial c} - \int M \nabla v(\xv, \yv) \cdot \nabla (\phi\circ \kappav^{-1}) (\xv, \yv) \diff \lambda(\xv, \yv)\\
    &= - \int M \nabla v(\xv, \yv) \cdot \Big( (\nabla \phi)(\kappav^{-1}(\xv, \yv) )^\top D \kappav^{-1}(\xv, \yv) \Big) \diff \lambda(\xv, \yv)\\
    &= - \int M \nabla v(\xv, \yv) \cdot \Big(\nabla \phi( \kappav^{-1} (\xv, \yv) )^\top A_{\kappav}^{\top} \Big) \diff \lambda(\xv, \yv )\\
    &= - \int M \nabla v( \kappa_1 \xv, \kappa_2 \yv) \cdot \Bigg(\nabla \phi(\xv, \yv)^\top A_{\kappav}^{\top} \Big) \diff \lambda(\xv, \yv)\\
    \overset{\rm (\cref{lem:inv_leb})}&{=} - \int  \nabla v( \kappa_1 \xv, \kappa_2 \yv)^\top M^\top A_{\kappav} \nabla \phi(\xv, \yv) \diff \lambda(\xv, \yv)\\
    &= - \int  \nabla v(\kappa_1 \xv, \kappa_2 \yv)^\top 
    \begin{bmatrix}
    	A_{\kappa_1} & \zero\\
    	\zero & \zero
    \end{bmatrix} \nabla \phi(\xv, \yv) \diff \lambda(\xv, \yv)\\
    &= - \int \big[\nabla_1 (v\circ \kappav)(\xv, \yv)^\top ~~ \zero \big]~\nabla \phi(\xv, \yv) \diff \lambda(\xv, \yv)\\
    &= - \int \big[\nabla_1 (v\circ \kappa)(\xv, \yv)^\top ~~ \zero \big]~\nabla \phi(\xv, \yv) \diff \lambda(\xv, \yv) + \Big[\phi(\xv, \yv)~\big[\nabla_1 (v\circ \kappav)(\xv, \yv)^\top ~~ \zero \big]\Big]_{\partial c}\\
    &=\int \phi(\xv, \yv)~~\nabla \cdot \big[\nabla_1 (v\circ \kappav)(\xv, \yv)^\top ~~ \zero \big] \diff \lambda(\xv, \yv) =\int \phi(\xv, \yv)~\Delta_1(v \circ \kappav)(\xv, \yv) \diff \lambda(\xv, \yv).
    \end{align*}
   Thus, $(\Delta_1 v) \circ \kappav (\xv, \yv) = \Delta_1(v \circ  \kappav ) (\xv, \yv)$. 
\end{proof}

\begin{lemma}
\label{lem:eqv_score_meas}
    $\mu(\cdot \vert \cdot )$ is $\Gv$-invariant if and only if
    \begin{align}
        \sv(\kappa_1\xv\vert \kappa_2 \yv) = A_{\kappa_1}~\sv(\xv\vert \yv) \label{eq:eqv_score_meas}
    \end{align}
    for all $\kappav \in \Gv$, $\xv \in \Rb^m$ and $\yv \in \Rb^n$, where $\sv(\xv \vert \yv)$ denotes the score function $\nabla_\xv\log p(\xv\vert \yv)$. 
\end{lemma}
\begin{proof}
($\Rightarrow$) By Lem~\ref{lem:inv_prob_density},  if $\mu$ is $\Gv$-invariant, its density $p(\xv \vert \yv) = p(\kappa_1\xv \vert \kappa_2 \yv)$. Taking $\log$ on both sides, followed by taking the derivative with respect to $\xv$ yields 
\begin{align}
    A_{\kappa_1}^\top~\sv(\kappa_1\xv\vert \kappa_2 \yv) = \sv(\xv\vert \yv),
\end{align}
as $D\kappa_1(\xv) = A_{\kappa_1}$.

\noindent
($\Leftarrow$) Conversely, \eqref{eq:eqv_score_meas} yields
\begin{align}
	p(\xv \vert \yv) = p(\kappa_1\xv \vert \kappa_2 \yv) + C,
\end{align}
where $C$ must be zero so that $p(\xv \vert \yv)$ and $p(\kappa_1\xv \vert \kappa_2 \yv)$ are valid densities. 
\end{proof}

Then we prove the following Lemma presented in the \cref{sec:group-invariance}.
\eqvScore*
\begin{proof}
The conditional density case is immediate given \cref{lem:inv_prob_density} and \cref{lem:eqv_score_meas} while the unconditional one is the special case that $n = 0$. 
\end{proof}

\begin{lemma}[\cite{SongDKKEP2021}]
\label{lem:fwd_ode}
Let $p_t$ be the marginal distribution of $\xv_t$ that satisfies SDE:
\begin{equation}
	\diff \xv_t = \fv(\xv_t, \yv, t)\diff t + g(t) \diff \ws_t, ~~ \xv_0\sim p_0(\xv_0 \vert \yv). 
\end{equation}
Besides, let $\sv_t(\cdot \vert \yv) = \nabla \log p_t(\cdot \vert \yv)$. Then, the ODE
\begin{align}
 \diff \xv = \tilde\fv(\xv, \yv, t)\diff t  \label{eq:reverse_ODE}
\end{align}
with
\begin{align}
    \tilde \fv (\xv, \yv, t) &= \fv(\xv, \yv,  t) - \frac{1}{2} g(t)^2 ~ \sv_t(\xv \vert \yv)  \label{eq:fwd_ODE:1}
\end{align}
also has the same marginal distribution $p_t$ for all $t \geq 0$. 
\end{lemma}

\begin{proof}
The marginal distribution $p_t(\xv \vert \yv)$ evolution is characterized by the Fokker-Planck equation \citep{Oksendal2003}:
\begin{align}
\frac{\partial p_t(\xv \vert \yv)}{\partial t} &= - \nabla \cdot \big(\fv(\xv, \yv, t) p_t(\xv\vert \yv)\big) + \frac{1}{2} \nabla \cdot \nabla \big(g(t)^2 p_t(\xv\vert \yv) \big)\\
&= - \sum_{i=1}^d\frac{\partial}{\partial x_i}[f_i(\xv, \yv,t) p_t(\xv\vert \yv)] + \frac{1}{2}\sum_{i=1}^d  \frac{\partial^2}{\partial x_i^2}[g(t)^2~p_t(\xv\vert \yv)] \\
&= - \sum_{i=1}^d \frac{\partial}{\partial x_i}\left\{[f_i(\xv,\yv, t) p_t(\xv\vert \yv)] - \frac{g(t)^2}{2}\big[  p_t(\xv\vert \yv)    \frac{\partial}{\partial x_i} \log p_t(\xv\vert \yv)\big] \right\}\\
&= - \sum_{i=1}^d\frac{\partial}{\partial x_i}\left[f_i(\xv,\yv,t) - \frac{g(t)^2}{2} \frac{\partial}{\partial x_i} \log p_t(\xv\vert \yv)\right] p_t(\xv\vert \yv),
\end{align}
where the last line is the Fokker-Planck equation of 
\begin{align}
 \diff \xv = \tilde\fv(\xv,\yv, t)\diff t\label{eq:general_reverse_SDE_family:4}
\end{align}
with $\tilde\fv(\xv, \yv, t)$ given in \eqref{eq:fwd_ODE:1}.
\end{proof}

Now we are ready to give the if and only if statement on the structurally preserving property of the distributions induced by 
\begin{align}
    \diff \xv_t = \fv(\xv_t, \yv, t) \diff t + g(t) \diff \ws_t \label{eq:appx:fwd_diff_process}
\end{align}
Notably, a sufficient condition given by \eqref{eq:eqv_drift_diff_meas} with the left-hand side equal to zero is firstly discussed by \cite{YimTDMDBJ2023}.

Then we give the equivalent conditions on the drift terms to ensure the structure-preserving property of SDE flows and its equivalence to the \cref{prop:eqv_drift} presented in the main text. 
\begin{proposition}
\label{prop:eqv_drift_meas}
 Given a diffusion process in \eqref{eq:appx:fwd_diff_process} with $\Gv$-invariant $\mu_0(\cdot \vert \cdot)$, $\mu_t(\cdot \vert \cdot)$ is $\Gv$-invariant for all $t\geq 0$  if and only if
\begin{align}
    A_{\kappa_1}^\top \fv(\kappa_1 \xv, \kappa_2 \yv, t) - \fv(\xv, \yv, t) \in [\zero]_{\mu_t}. \label{eq:eqv_drift_diff_meas}
\end{align}
for all $t > 0$, $\xv \in \Rb^m$, $\yv \in \Rb^n$ and $\kappav \in \Gv$.
\end{proposition}
\begin{proof}
$(\Rightarrow)$ Let $\tilde \fv$ denote the corresponding ODE drift shown in \eqref{eq:fwd_ODE:1}. Assume $\mu_t(\cdot \vert \cdot)$ is $\Gv$-invariant for all $t\geq 0$. Then by \cref{lem:eqv_drift}, for all $\kappav \in \Gv$,  the ODE drift $\tilde \fv$ satisfies
\begin{align}
    A_{\kappa_1}^\top \fv(\kappa_1 \xv, \kappa_2 \yv, t) - \fv(\xv, \yv, t) \in [\zero]_{\mu_t}.
\end{align}
That is,
\begin{align}
A_{\kappa_1}^\top \Big(\fv(\kappa_1\xv,\kappa_2 \yv, t) - \frac{1}{2}g(t)^2  \sv_t(\kappa_1\xv \vert \kappa_2\yv)\Big) - \Big(\fv(\xv, \yv, t) - \frac{1}{2}g(t)^2 \sv_t(\xv\vert \yv)\Big) \in [\zero]_{\mu_t}.
\end{align}
By \cref{lem:eqv_score_meas}, $\Gv$-invariance of $\mu_t$ implies that $A_{\kappa_1}^\top\sv_t(\kappa_1\xv\vert \kappa_2\yv) = \sv_t(\xv\vert \yv)$. Thus, \eqref{eq:eqv_drift_diff_meas} follows. 

\noindent
$(\Leftarrow)$ Assume \eqref{eq:eqv_drift_diff_meas} holds. For $\phi \in C_c^\infty(\Rb^{m+n})$ and $\kappav \in \Gv$, let $\psi = \phi \circ \kappav$,  and then we have
\begin{align*}
	& \frac{\diff}{\diff t}\int \phi( \xv, \yv) \diff\muv_t(\kappa_1^{-1} \xv \vert \kappa_2^{-1} \yv) = \frac{\diff}{\diff t} \int \psi(\xv, \yv) \diff \muv_t(\xv, \yv)\\
	\overset{\eqref{eq:liouville_meas_cond}}&{=} \int \nabla_1\psi(\xv, \yv)^\top \tilde\fv(\xv,\yv, t) \diff \mu_t(\xv \vert \yv) = \int \nabla_1\psi( \xv, \yv)^\top \Big(\fv(\xv, \yv, t) - \frac{1}{2}g^2(t)\sv_t(\xv\vert \yv) \Big) \diff \mu_t(\xv\vert \yv)\\
	&= \underbrace{\int \nabla_1\psi(\xv, \yv)^\top \fv(\xv, \yv, t) \diff \mu_t(\xv\vert \yv)}_{\bf I} - \frac{1}{2}g^2(t)\underbrace{\int\nabla_1\psi(\xv, \yv)^\top \sv_t(\xv \vert \yv) \diff \mu_t(\xv \vert \yv)}_{\bf II}. \numberthis \label{eq:prop:eqv_drift_meas:1}
\end{align*}
By \eqref{eq:eqv_drift_diff_meas} and applying the same argument to derive (\textbf{A}) in the proof of \cref{lem:eqv_drift}. We have
\begin{align*}
    \int \nabla_1\phi(\kappa_1 \xv, \kappa_2 \yv)^\top \fv(\kappa_1 \xv,\kappa_2\yv, t) \diff \mu_t(\xv\vert \yv) = \int \nabla_1\psi(\xv, \yv)^\top \fv(\xv, \yv, t) \diff \mu_t(\xv\vert \yv) = {\bf I}
\end{align*}
Then,
\begin{align*}
    {\bf I} =& \int \nabla_1\phi(\xv, \yv)^\top \fv(\xv,\yv, t) \diff \mu_t(\kappa_1^{-1}\xv\vert \kappa_2^{-1} \yv) \\
    =& - \int \phi(\xv, \yv) ~ \nabla_\xv \cdot\big(p_t(\kappa_1^{-1}\xv \vert \kappa_2^{-1}\yv ) \fv(\xv,\yv, t) \big) \diff \lambda(\xv).
\end{align*}
In addition,
\begin{align*}
    {\bf II} &= \int \nabla_1 \psi(\xv, \yv)^\top p_t(\xv \vert \yv) \diff \lambda(\xv) = -\int \psi(\xv, \yv) \Delta_1 p_t(\xv \vert \yv) \diff \lambda(\xv)\\
    &= - \int \phi(\xv, \yv) \Delta_1 p_t(\kappa^{-1}_1 \xv \vert \kappa^{-1}_2 \yv) \diff \lambda(\xv) \overset{\rm (Lem~\ref{lem:inv_laplacian})}{=} - \int \phi(\xv, \yv) ~ \Delta_1 (p_t\circ\kappav^{-1})(\xv \vert \yv) \diff \lambda(\xv)
\end{align*}
As a result, by Eq~\eqref{eq:prop:eqv_drift_meas:1}, we have
\begin{align*}
	&\frac{\diff}{\diff t}\int_c \phi(\xv, \yv) p_t(\kappa_1^{-1}\xv\vert \kappa_2^{-1}\yv ) \diff \lambda(\xv) = \frac{\diff}{\diff t}\int \phi( \xv, \yv) \diff\muv_t(\kappa_1^{-1} \xv, \kappa_2^{-1} \yv)\\
	 = &- \int_c \phi(\xv, \yv) \left[\nabla_\xv \cdot \big(p_t(\kappa_1^{-1}\xv \vert \kappa_2^{-1}\yv ) \fv(\xv,\yv, t)\big) - \frac{1}{2}g^2(t) \Delta_1 (p_t \circ \kappav^{-1}) (\xv, \yv) \right]\diff \lambda(\xv)
\end{align*}
Hence,
\begin{align}
	\frac{\diff}{\diff t}p_t(\kappa_1^{-1}\xv\vert \kappa_2^{-1}\yv)  = - \nabla_\xv \cdot \big(p_t(\kappa_1^{-1}\xv \vert \kappa_2^{-1}\yv ) \fv(\xv,\yv, t) \big) - \frac{1}{2}g^2(t) \Delta_1 (p_t \circ \kappav^{-1}) (\xv, \yv). 
\end{align}
By the Fokker-Planck equation, we also have
\begin{align}
	\frac{\diff}{\diff t}p_t(\xv \vert \yv)  = - \nabla_\xv \cdot(p_t(\xv\vert \yv) \fv(\xv, \yv, t)) + \frac{1}{2}g^2(t)~(\Delta_1 p_t) (\xv\vert \yv).
\end{align}
Therefore, $p_t = p_t \circ \kappav^{-1}$, which, by Lem~\ref{lem:inv_prob_density}, implies $\muv_t$ is $\Gv$-invariant. 
\end{proof}

\eqvDrift*
\begin{proof}
\cref{prop:eqv_drift} is equivalent to \cref{prop:eqv_drift_meas} but in probability density notations by \cref{lem:inv_prob_density}.
\end{proof}

Finally, we show how our theoretical results can be applied to characterize the structure-preserving properties of the diffusion bridges. The main results are given in \cref{prop:appx:equiv_bridge} and are collectively presented with the counterparts for the regular diffusion processes in \cref{cor:sp_diff_linear_op}.
\begin{lemma}
\label{lem:eqv_pushforward}
Let $p_t$ denote the distribution of $\xv_t$ generated by SDE:
	\begin{align}
		\diff \xv_t = \uv(\xv_t, t) \diff t + g(t) \diff \ws_t. 
	\end{align}
	Then $p(\kappa \xv_T \vert \kappa \xv_t ) = p(\xv_T \vert \xv_t )$ for all $\kappa \in \Gc_r$ and $T \geq t$ if 
	\begin{align}
		    A_\kappa^\top \uv(\kappa\xv, t) - \uv(\xv, t)  \in [\zero]_{\mu_t}, \label{eq:eqv_pushforward:1}
	\end{align}
	for all $\kappa \in \Gc$.
\end{lemma}
\begin{proof}
	Without loss of generality, it is sufficient to show 
	\begin{align}
		p(\kappa \xv_t \vert \kappa \xv_0 ) = p(\xv_t \vert \xv_0 )
	\end{align}
	for all $\kappa \in \Gc$ and $t \geq 0$. Then let $\yv = \xv_0$, $\fv(\xv_t, \yv, t) = \uv(\xv, t)$ and $\Gv = \{(\kappa, \kappa) \vert \kappa \in \Gc_r\}$. As \eqref{eq:eqv_pushforward:1} implies $A_\kappa^\top \uv(\kappa\xv, t) - \uv(\xv, t) \in [\zero]_{\mu_t}$. Then, combined with \eqref{eq:eqv_pushforward:1}, Prop~\ref{prop:eqv_drift_meas} shows $\mu_t(\xv_t \vert \yv)$ is $\Gv$-invariant. By Lem~\ref{lem:inv_prob_density}, we have $p(\kappa\xv_t \vert \kappa \xv_0) = p(\xv_t \vert \xv_0)$, which completes the proof.  
\end{proof}

\begin{lemma}
\label{prop:appx:equiv_bridge}
Assume the two ends $(\xv_0, \xv_T) \in \Rb^d \times \Rb^d$ of the diffusion bridges follow a $\Gv$-invariant conditional distribution $\muv_{0\vert T}(\xv_0 \vert \xv_T)$, where $\Gv = \{(\kappa, \kappa) \vert \kappa \in \Gc\}$. Let $\muv_{t\vert T}$ denote the measure of $(\xv_t, \xv_T)$ induced by diffusion bridge:
	\begin{align}
		\diff \xv_t = \big(\uv(\xv_t, t)  + g(t)^2 \hv(\xv_t, t, \xv_T, T)\big)\diff t + g(t) \diff \ws_t,
	\end{align}
	where $\hv(\xv_t, t, \xv_T, T) = \nabla_{\xv_t} \log p(\xv_T | \xv_t)$ is the gradient of the log transition kernel from $t$ to $T$ generated by the original SDE:
	\begin{align}
		\diff \xv_t = \uv(\xv_t, t) \diff t + g(t) \diff \ws_t. 
	\end{align}
	If 
	\begin{align}
		A_\kappa^\top \uv(\kappa\xv, t) - \uv(\xv, t) =  \zero
	\end{align}
	for all $\kappa \in \Gc$, then $\uv(\xv, t) + g(t)^2 \hv(\xv, t, \yv, T)$ satisfies \cref{eq:eqv_drift_diff_meas} and  
 $\mu_{t\vert T}$ is $\Gv$-invariant for all $t \in [0,T]$. 
\end{lemma}
\begin{proof}
	By Lem~\ref{lem:eqv_pushforward}, we have $p(\kappa \xv_T \vert \kappa \xv_t ) = p(\xv_T \vert \xv_t )$ for all $\kappa \in \Gc$. As a result, 
	\begin{align}
		A_{\kappa}\,\nabla_{\xv_t} \log p(\xv_T \vert \xv_t) = \nabla_{\kappa\xv_t} \log p(\kappa \xv_T \vert \kappa \xv_t). 
	\end{align}
	Or equivalently,
	\begin{align}
		 \hv(\xv_t, t, \xv_T, T) =  A_{\kappa}^\top \hv(\kappa\xv_t, t, \kappa\xv_T, T). 
	\end{align}
	As a result,
	\begin{align}
		\fv(\xv, \yv, t) = \uv(\xv, t) + g(t)^2 \hv(\xv, t, \yv, T). 
	\end{align}
	satisfies \eqref{eq:eqv_drift_diff_meas}, and thus \cref{prop:eqv_drift_meas} implies $\muv_{t \vert T}$ is $\Gv$-invariant.
\end{proof}

\SPDiffLO*
\begin{proof}
    The first part of the proposition regarding the unconditional $p_t$ follows \cref{prop:eqv_drift_meas} while the second part basically restates \cref{prop:appx:equiv_bridge} in density notation where their equivalence can be seen by \cref{lem:inv_prob_density}.
\end{proof}

\section{GROUP INVARIANT WEIGHT TIED CONVOLUTIONAL KERNELS}
\label{appx:weight-tied-kernel}

As stated in \cref{sec:model-architectures}, we currently limit our attention to linear groups $\Gc_\Lc$. In this setting, we can directly impose $\Gc_\Lc$-equivariance into the diffusion model by constructing specific CNN kernels. 

In particular, for a given linear group $\Gc_\Lc$ we can construct a group equivariant convolutional kernel $\ks\in\mathbb{R}^{d\times d}$, of the form
\begin{align}
    \ks = \begin{array}{|c|c|cc|c|}
        \hline
        k_{1,1} & k_{1,2} & \cdots & & k_{1,d} \\
        \hline
        \vdots & \vdots & \ddots & & \vdots \\
        \vdots & \vdots & & & \vdots \\
        \hline
        k_{d-1,1} & k_{d-1,2 }& \cdots &  & k_{d-1,d} \\
        \hline
        k_{d,1} & k_{d,2} & \cdots & & k_{d,d}\\
        \hline
    \end{array}\, ,
\end{align}

such that 
\begin{equation*}
    \hs(\ks*\xv) = \ks*\hs(\xv)
\end{equation*}
for any $h\in\Gc_\Lc$ and $\xv\sim p_{data}$ by constraining the individual kernel values to obey a system of equalities set by the group invariance condition 
\begin{equation}
    \hs(\ks) = \ks.
\end{equation}

\paragraph{Example: Vertical Flipping.} A concrete example,  which was discussed in \cref{sec:group-invariance}, is to consider the group $\Gc = \{\fs_x, \ev\}$ where $\fs_x$ is a vertical flipping  operation with $\fs_x^{-1} = \fs_x$. A convolutional kernel $\ks\in\mathbb{}R^{3\times 3}$ constrained to be equivariant to actions from this group would take the form:
\begin{align}
    \ks = \begin{array}{|c|c|c|}
        \hline
        d & a & d \\
        \hline
        e & b & e \\
        \hline
        f & c & f\\
        \hline
    \end{array}
\end{align}
It should be clear given the form of $\ks$ that 
\begin{equation*}
    \fs_x(\ks*\xv) = \ks*\fs_x(\xv)
\end{equation*}
and consequently also for $\fs_x^{-1}$, as desired. Likewise, we also present the weight-tied kernels for C4 and D4. 

\paragraph{Example: The $C_4$ Cyclic and $D_4$ Dihedral Group.} Recall that the $C_4$ cyclic group is composed of planar $90\deg$ rotations about the origin, and can be denoted as $C_4 = \{\es, \rs_1, \rs_2, \rs_3\}$ where $\rs_i$ represents a rotation of $i\times 90\deg$. Taking a convolutional kernel $\ks\in\mathbb{R}^{5\times 5}$ and constraining it to be $C_4$-equivariant results in $\ks$ being of the form:
\begin{align}
    \ks = \begin{array}{|c|c|c|c|c|}
        \hline
        a & b & c & d & a \\
        \hline
        d & e & f & e & b \\
        \hline
        c & f & g & f & c \\
        \hline
        b & e & f & e & d \\
        \hline
        a & d & c & b & a \\
        \hline
    \end{array}\, .
\end{align}
The $D_4$ dihedral group can then be ``constructed'' from $C_4$ by adding the vertial flipping operation from the past example; that is, $D_4=\{\es, \rs_1, \rs_2, \rs_3, \fs_x, \fs_x\circ\rs_1,  \fs_x\circ\rs_2,  \fs_x\circ\rs_3\}$. This requires further constraints to $\ks$ so that  
\begin{align}
    \ks = \begin{array}{|c|c|c|c|c|}
        \hline
        a & b & c & b & a \\
        \hline
        b & e & f & e & b \\
        \hline
        c & f & g & f & c \\
        \hline
        b & e & f & e & b \\
        \hline
        a & b & c & b & a \\
        \hline
    \end{array}\, .
\end{align}

Naturally, constraining convolutional kernels in this fashion has the advantage of reducing the number of model parameters -- with a possible loss in expressiveness when the kernel size is relatively small in comparison to the size of the group and structure of the data. For a more general discussion on $\Gc$-equivariant convolutional kernels in the context of CNNs we refer the reader to \citet{CohenW2016} and \citet{KniggeRB2022}.

\section{EQUIVARIANCE REGULARIZATION}
\label{appx:equivariance-regularization}
Instead of achieving $\Gc$-equivalence by adopting specific model architectures, as described in \cref{sec:model-architectures}, or by frame averaging \cite{PunyASMGHL2022}, we can also directly add a regularizer to the score-matching loss to inject this preference. Specifically, according to \cref{lem:eqv_score}, the estimated score $\sv_{\thetav}(\cdot , t)$ is equivariant if 
\begin{align}
    \sv_{\thetav}(\kappa\xv, \kappa\yv,t) = \kappa \sv_{\thetav}(\xv, \yv, t),
\end{align}
for all $\kappa\in \Gc$. (For the unconditional distribution, similar techniques can be applied by omitting the second argument of $\sv_{\thetav}$.) Thus, we propose the following regularizer to encourage the two terms to match for all $\xv$ and $t$:
\begin{align}
\label{eq:score_matching_regularization}
	\!\!\!\Rc({\thetav},\bar{{\thetav}})\! &= \!\Eb\! \left[ \!\frac{1}{|\Gc|}\!\!\sum_{\kappa\in\Gc}\!\big\| \sv_{\thetav}(\kappa \xv, \kappa  \yv, t) - \kappa  \sv_{\bar{{\thetav}}}(\xv, \yv, t) \big\|^2 \! \right]
\end{align}
where the expectation is taken over the same variables in the regular score-matching loss and $\bar{\thetav}$ denotes the exponential moving average (EMA) of the model weights
\begin{align}
\label{eq:ema}
    \bar{\thetav} \leftarrow {\rm stopgrad}(\mu \bar {\thetav} + (1-\mu) {\thetav})~~{\rm with}~~\mu\in [0,1),
\end{align}
which helps improve training stability. In practice, iterating over all elements in $\Gc$ may be intractable. Thus, for each optimization step, $\Rc(\thetav, \bar\thetav)$ is one-sample approximated by:
\begin{align}
\label{eq:score_matching_regularization_est}
	\Rc({\thetav},\bar{{\thetav}})&\approx \Eb \left[ \big\| \sv_{\thetav}(\kappa\xv, \kappa\yv, t) - \kappa \sv_{\bar{{\thetav}}}(\xv, \yv, t) \big\|^2 \right], 
\end{align}
with randomly picked $\kappa \in \Gc$.

\section{CONSTRUCTIONS OF EQUIVARIANT NOISY SEQUENCE}
\label{appx:eqv_aligner}

In this section, we present a method to construct an equivariant noisy sequence $\{\epsilonv_i\}_{i=1}^n$ with respect to some $\xv_n \sim  q(\xv)$ without knowing the ``true'' orientation of $\xv_n$.

Let $q$ denote the distribution of $\xv_n$. Construct a function $\phi:\Rb^d \rightarrow \Rb^d$ such that: (1) for all $\kappa \in \Gc$, $\xv\sim q$ or $\xv\sim \Nc(\zero, I)$, $\phi(\kappa \xv) = \kappa \phi(\xv)$ almost surely; (2) for all $\xv, \yv \in \Rb^d$, there exists a unique $\kappa \in \Gc$ such that $ \phi(\xv) = \kappa\phi(\yv)$. For example, for $\Rb^2$ with $\Gc$ consisting of element-swapping operators, $\phi$ can be the function that outputs one-hot vector indicating the max element of the input. We will present some selections of $\phi$ for common $\Gc$ below.

Given starting point $\xv_n$ and a noise sequence $\{\epsilonv_i\}_{i=1}^n$, choose $\kappa \in \Gc$ such that $\phi(\xv_n) =\kappa\phi(\epsilonv_n)$. Then we use the noise sequence $\tilde\epsilonv_i = \kappa\epsilonv_i$ for the evaluation of \cref{eq:euler_maruyama}. To see why this approach works, assume that $\xv_n$ is updated to $\rs \xv_n$ for some $\rs\in \Gc$. Then, $\phi(\rs \xv_n) = \rs \phi(\xv_n) = (\rs \circ \kappa)\phi(\epsilonv_n)$, and thus the sequence becomes $\{\rs \circ \kappa \, \epsilonv_i\}_{i=1}^n = \{\rs  \tilde \epsilonv_i\}_{i=1}^n$. Note that this is a general method to create an equivariant noise sequence with respect to any input.

Below, we present some choices of $\phi$ for some common linear operator groups for 2D images.

\paragraph{Example: Vertical Flipping.} The function $\phi_v$ can be chosen to output either of two images:
\begin{center}
    \includegraphics{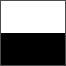}
    \includegraphics{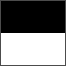}
\end{center}
Specifically, if the input image $\xv$ has the max value on the upper half of the image, $\phi$ returns the left plot; otherwise, the right one. (Here, we assume that it is almost surely that the max value cannot appear in both halves.)

It is obvious that if the input is flipped vertically, the output will be flipped in the same way. Therefore, the first condition is satisfied. For the second, if $\phi_v(\xv)$ and $\phi_v(\yv)$ have the same output, $\kappa$ is the identity operator; otherwise, $\kappa$ is the vertical flipping. For multichannel input, $\phi$ can be applied independently to each channel. 

Applying the same idea, we can derive the corresponding $\phi_h$ for horizontal flipping. 

\paragraph{Example: $C_4$ Cyclic Group.} We can use a similar idea to derive $\phi_{C4}$ for $C_4$ cyclic that is composed of planar $90\deg$ rotations about the origin. In this case, $\phi_{C4}$ has four possible outputs
\begin{center}
    \includegraphics{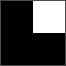}
    \includegraphics{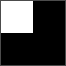}
    \includegraphics{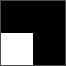}
    \includegraphics{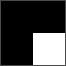}
\end{center}
such that $\phi_{C4}$ assigns the quadrant white if the input has the max value in that quadrant. (Here, we assume it is almost surely that the max value cannot appear in multiple quadrants.) Then, it is straightforward to see that $\phi_{C4}$ satisfies the two conditions of $\phi$.

\paragraph{Example: $D_4$ Dihedral Group.} As we have mentioned in \cref{appx:weight-tied-kernel}, the $D_4$ dihedral group can be ``constructed'' from $C_4$ by adding the vertial flipping operation. As a result, we can combine $\phi_v$ and $\phi_{C4}$ to construct the corresponding $\phi_{D4}$ for $D_4$. Assume that $\phi_v$ assigns one to the elements corresponding to the white pixels and zero to the ones associated with the black. Likewise, let $\phi_{C4}$ assign two to the elements corresponding to the white pixels and zero to the rest. Then we define $\phi_{D4} = \phi_{v} + \phi_{C4}$. It is easy to check that both $\phi_{v}$ and $\phi_{D4}$ satisfy the first condition of $\phi$ for all $\kappa\in D_4$ (i.e., vertical flipping, rotation, and their composition). We also note that the range $\phi_{D4}$ contains eight distinct elements. Starting from one element, we get all the elements by applying one of the eight operators in $D_4=\{\es, \rs_1, \rs_2, \rs_3, \fs_x, \fs_x\circ\rs_1,  \fs_x\circ\rs_2,  \fs_x\circ\rs_3\}$, which suggests $\phi_{D4}$ satisfies the second condition.

\section{DATASET DETAILS}
\label{appx:datasets}

This section contains detailed discussion on the contents and preprocessing of each dataset mentioned in \cref{sec:emp_study}.

\subsection{Rotated MNIST}
Rotated MNIST dataset \citep{LarochelleECBB2007} contains random \(90^\circ\) rotations of MNIST images \citep{Deng2012}, resulting in a \(C_4\)-invariant distribution. This dataset was generated following the description in \cite{KniggeRB2022}, and has been commonly used to evaluate group-invariant CNN models, as seen in \citep{DeyCG2021, BirrellKRLZ2022}, with experiments on 1\% (600), 5\% (3000), and 10\% (6000) of the dataset.

\subsection{LYSTO}
The LYSTO dataset \citep{JiaoLYSTO2023} consists of 20,000 labeled image patches at a resolution of 299x299x3 extracted at 40X magnification from breast, colon, and prostate cancer samples stained with CD3 or CD8 dyes. This data was preprocessed by first scaling all the images to 128x128x3 before randomly sampling 64x64x3 image patches from each image to generate the final dataset. The data was first scaled to increase the feature density within each randomly sampled patch. The resulting data exhibits \(D_4\) invariance due to natural rotational and mirror invariance.

\subsection{LYSTO Denosing}
To construct the LYSTO denosing dataset, we take the LYSTO 64x64x3 patches, generated as described above, and downscale them to \(1/4\) the resolution and then upscaled back using {\rm LANCZOS} interpolation to form conditional training pairs.

\subsection{ANHIR}
The ANHIR dataset \citep{BorovecEL2020} provides whole-slide images of lesions, lung-lobes, and mammary-glands at a variety of different resolutions, from 15kx15k to 50kx50k. We only make use of the lung images. This data was processed following the method outlined in \cite{DeyCG2021} from which random 64x64x3 image patches are extracted.

\subsection{CT-PET}
The CT-PET dataset \citep{GatidisHFFNPSKCR2022} includes 1014 (501 positives and 512 controls) annotated whole-body paired FDG-PET/CT scans, comprised of 3D voxels, of patients with malignant lymphoma, melanoma, and non-small cell lung cancer. We extract 40 slices per voxel to construct the training datasets. This data is restricted under TCIA restricted lience. Formal access must be filed for and granted before the data can be made available to practitioners from the  \href{https://www.cancerimagingarchive.net}{Cancer Imaging Archive} (CIA). Scripts for processing this data, such as into slices, are provided by CIA on the dataset page.

The style-transfer dataset was constructed by first slicing the 3D voxel volumes of the patients into 2D images of the middle of the patients. These images were then cropped to just contain the torso and head and then scaled to 256x256x3. A patient's CT scan slice was then paired with the matching PET scan image slice to form the final dataset. 

\section{MODEL DETAILS}
\label{appx:Hyperparameters}

\begin{table}[ht]
    \vspace{-0.7em}
    \centering
    \caption{Model summary. VP-SDE denotes the regular diffusion model with variance preservation configuration in \cref{tb:fg_selection}. 
    \cmark: Theo. guaranteed ~ \xmark: Not theo. guaranteed ~ --:~Not Applicable }
    \label{tab:SPDiff-model-summary}
    \resizebox{0.6\columnwidth}{!}{%
    \begin{tabular}{llccccccc}
        \toprule 
        Model & Arch. &  $\Gc_\Lc$-inv Smpl & Eqv Traj \\ \midrule
        VP-SDE & U-Net & \xmark & \xmark\\
        SPDM+WT & U-Net (WT) & \cmark & \cmark\\
        SPDM+FA &  U-Net  & \cmark & \cmark\\
        DDBM    & U-NET & \xmark & \xmark \\
        SPDM+FA (Bridge)    & U-NET & \cmark & \cmark \\
        SP-GAN  & CNN & \cmark & -- \\ 
        GE-GAN  & CNN & \xmark & -- \\
        Pix2Pix & U-NET \& CNN & \xmark & -- \\
        ${\rm I^2SB}$ & U-NET & \xmark & \xmark \\
        \bottomrule
    \end{tabular}
    }
    \label{tab:model-summary}
\end{table}

The implementation details and hyperparamters used while training the models presented in \cref{sec:emp_study} over the listed datasets are given below. \cref{tab:model-summary} provides a summary of all the models discussed within the paper and their theoretical invariance (equivariance) guarantees. 

We refer the reader to the reader to Appx.G. of \citet{BirrellKRLZ2022} for a more in-depth discussion of the implementation details and training parameters used to produce the results of SP-GAN reported in \cref{sec:emp_study}. We will only summarize the training parameters we changed from the defaults given in the forgoing reference.  

All Diffusion models (VP-SDE, SPDM+WT, SPDM+FA, SPDM+FA(Bridge)\,) are trained using the Adam optimizer \cite{KingmaB14} with learning rate $\eta = 0.0001,\, 0.0002$, and weight decay rate of $\gamma = 0.0$; separately, we make use of the exponential moving average (EMA) of the model weights \cref{eq:ema} with $\mu= 0.999,\, 0.9999,\, 0.9999432189950708$, which are the values commonly used when training this style of diffusion model.

\subsection{Rotated MNIST}
For the Rotated MNIST datasets, the diffusion models are configured using the following model parameters: dropout rate $\ds = 0.1$.

\begin{table}[ht]
\centering
\resizebox{0.95\columnwidth}{!}{%
\centering
    \begin{tabular}{@{}lllllccccc@{}}
    \toprule
    Model & Loss & Batch & Cond. & Aug & Attn. res. & Num. Ch. & Num. Heads & Ch. Scal. & Scale Shift \\ \midrule
    SPDM & $L_2$ & 32 & True & True & 8 & 128 & 16 & 1,2,2 & True \\
    SPDM+WT & $L_2$ & 32 & True & False & 8 & 128 & 16 & 1,2,2 & True \\
    SPDM+FA & $L_2+FA$ & 32 & True  & True & 8 & 128 & 16 & 1,2,2 & True \\
    \midrule
    \end{tabular}
}
\end{table}

The GAN based methods were trained using the scripts provided by \cite{DeyCG2021} and \cite{BirrellKRLZ2022} with the following settings:

\begin{table}[!ht]
\centering
\resizebox{0.7\columnwidth}{!}{%
\centering
    \begin{tabular}{@{}lllllccccc@{}}
    \toprule
    Model & Loss & Batch & Cond. & Aug & latent dim & gp weight & lr & alpha  \\ \midrule
    SP-GAN & $D_2^L$ & 64 & True & True & 64 & 10.0 & $1e{-4}$ & 2 \\
    GE-GAN & $RA$ & 64 & True & True & 64 & 10.0 & $1e{-4}$ & -- \\ \midrule
    \end{tabular}
}
\end{table}

\subsection{LYSTO}
For the LYSTO dataset, the diffusion models are configured using the following model parameters: dropout rate $\ds = 0.1$.

\begin{table}[!ht]
\centering
\resizebox{0.95\columnwidth}{!}{%
\centering
    \begin{tabular}{@{}lllllccccc@{}}
    \toprule
    Model & Loss & Batch & Cond. & Aug & Attn. res. & Num. Ch. & Num. Heads & Ch. Scal. & Scale Shift \\ \midrule
    VP-SDE & $L_2$ & 32 & True & True & 32,16,8 & 128 & 64 & 1,2,2,2 & True \\
    SPDM+WT & $L_2$ & 32 & True & False & 32,16,8 & 128 & 64 & 1,2,2,2 & True \\
    SPDM+FA & $L_2+FA$ & 32 & True & True & 32,16,8 & 128 & 64 & 1,2,2,2 & True \\
    \midrule
    \end{tabular}
}
\end{table}

The GAN based methods were trained using the scripts provided by \cite{DeyCG2021} and \cite{BirrellKRLZ2022} with the following settings:

\begin{table}[!ht]
\centering
\resizebox{0.7\columnwidth}{!}{%
\centering
    \begin{tabular}{@{}lllllccccc@{}}
    \toprule
    Model & Loss & Batch & Cond. & Aug & latent dim & gp weight & lr & alpha  \\ \midrule
    SP-GAN & $D_2^L$ & 32 & True & True & 128 & 10.0 & $1e{-4}$ & 2 \\
    GE-GAN & $RA$ & 32 & True & True & 128 & 10.0 & $1e{-4}$ & -- \\ \midrule
    \end{tabular}
}
\end{table}

\subsection{ANHIR}
For the ANHIR dataset, the diffusion models are configured using the following model parameters: dropout rate $\ds = 0.1$.

\begin{table}[!ht]
\centering
\resizebox{0.95\columnwidth}{!}{%
\centering
    \begin{tabular}{@{}lllllccccc@{}}
    \toprule
    Model & Loss & Batch & Cond. & Aug & Attn. res. & Num. Ch. & Num. Heads & Ch. Scal. & Scale Shift \\ \midrule
    VP-SDE & $L_2$ & 32 & True & True & 32,16,8 & 128 & 64 & 1,2,2,2 & True \\
    SPDM+WT & $L_2$ & 32 & True & False & 32,16,8 & 128 & 64 & 1,2,2,2 & True \\
    SPDM+FA & $L_2+FA$ & 32 & True & True & 32,16,8 & 128 & 64 & 1,2,2,2 & True \\ \midrule
    \end{tabular}
}
\end{table}

The GAN based methods were trained using the scripts provided by \cite{DeyCG2021} and \cite{BirrellKRLZ2022} with the following settings:

\begin{table}[!ht]
\centering
\resizebox{0.7\columnwidth}{!}{%
\centering
    \begin{tabular}{@{}lllllccccc@{}}
    \toprule
    Model & Loss & Batch & Cond. & Aug & latent dim & gp weight & lr & alpha  \\ \midrule
    SP-GAN & $D_2^L$ & 32 & True & True & 128 & 10.0 & $1e{-4}$ & 2 \\
    GE-GAN & $RA$ & 32 & True & True & 128 & 10.0 & $1e{-4}$ & -- \\ \midrule
    \end{tabular}
}
\end{table}

\subsection{LYSTO Denoising Task}
For the LYSTO denoising dataset, the diffusion models are configured using the following model parameters: 

\begin{table}[!ht]
\centering
\resizebox{0.95\columnwidth}{!}{%
\centering
    \begin{tabular}{@{}lllllccccc@{}}
    \toprule
    Model & Loss & Batch & Cond. & Aug & Attn. res. & Num. Ch. & Num. Heads & Ch. Scal. & Scale Shift \\ \midrule
    VP-SDE & $L_2$ & 32 & True & True & 32,16,8 & 128 & 64 & 1,2,2,2 & True \\
    SPDM+FA & $L_2+FA$ & 32 & True & True & 32,16,8 & 128 & 64 & 1,2,2,2 & True \\
    Pix2Pix  & GAN & 32 & True & True & -- & -- & -- & -- &--  \\
    ${\rm I^2SB}$ & $L_2$ & 64 & True & True & -- & -- & -- & -- & -- \\ \midrule
    \end{tabular}
}
\end{table}

The entries in the above table are left empty for both Pix2Pix and $I^2SB$ as their architectures differ from the other models. In particular, Pix2Pix being a GAN makes use of a U-NET for the generator and custom discriminator architecture. The model comes with several predefined U-NET architecture configurations that can be selected; however, it is lacking a configuration for 64x64 images. We defined a suitable configuration by modifying that provided for 128x128 by reducing the number of down-sampling layers from $7$ to $4$. All other settings are left as default. The $I^2SB$ model by default it make use of a preconfigured U-NET architecture which it downloads from \href{https://openaipublic.blob.core.windows.net/diffusion/jul-2021/256x256_diffusion_uncond.pt}{openai}. For this task we simply make use of the default configuration settings without any modification.

\subsection{CT-PET Style Transfer Task}
For the CT-PET dataset, all models except Pix2Pix are trained in latent space, where the original images are first encoded by a fine-tuned pretrained VAE from stable diffusion \citep{RombachBLEO2022}. FA was applied during the fine-tuning and inference to ensure equivariance. The VAE takes the 256x256x3 images and encodes them into a 32x32x4 latent space representation. 

Models are configured using the following model parameters: 

\begin{table}[ht]
\centering
\resizebox{0.95\columnwidth}{!}{%
\centering
    \begin{tabular}{@{}lllllccccc@{}}
    \toprule
    Model & Loss & Batch & Cond. & Aug & Attn. res. & Num. Ch. & Num. Heads & Ch. Scal. & Scale Shift \\ \midrule
    DDBM & $L_2$ & 32 & True & True & 32,16,8 & 128 & 64 & 1,2,2,2 & True \\
    SPDM+FA & $L_2+FA$ & 32 & True & True & 32,16,8 & 128 & 64 & 1,2,2,2 & True \\
    Pix2Pix  & GAN & 32 & True & True & -- & -- & -- & -- & -- \\
    $\rm I^2SB$ & $L_2$ & 64 & True & True & 32,16,8 & 128 & 64 & 1,2,2,2 & True\\ \midrule
    \end{tabular}
}
\end{table}

The entries in the above table are left empty for both Pix2Pix as the architecture differs from the other models. Pix2Pix makes use of a U-NET for the generator and custom discriminator architecture. The model comes with a predefined U-NET architecture for 256x256x3 images which we use. All other settings are left as default. As discussed above the $I^2SB$ model by default it make use of a preconfigured U-NET architecture. We can't make use of the default architecture for this task as it is incompatible with the latent space embedding produced by the VAE. Instead we configure the U-NET architecture in an identical fashion to SPBM.

\subsection{Computational Resources}
\label{appx:model_resources}
All the model results reported within \cref{sec:emp_study} were trained using NVIDIA A40 or L40S equivalent. Training times for each model are reported in \cref{tab:model_training_times}.

\begin{table}[h]
    \centering
    \renewcommand{\arraystretch}{1.1} 
    \ttabbox{
    \resizebox{0.85\columnwidth}{!}{%
    \begin{tabular}{@{\hspace{2em}}lcccc@{\hspace{2em}}lccc@{}}
        \toprule
        \multicolumn{4}{c}{\textbf{Training Times for LYSTO \& ANHIR}} 
        & \multicolumn{4}{c}{\textbf{Training Times for LYSTO \& CT-PET}} \\
        \cmidrule(lr){1-4} \cmidrule(lr){5-8}
        \text{Model} & \text{GPUs} & \text{LYSTO} & \text{ANHIR} 
        & \text{Model} & \text{GPUs} & \text{LYSTO} & \text{CT-PET} \\
        \cmidrule(lr){1-4} \cmidrule(lr){5-8}
        VP-SDE   & 2 & 5 days   & 5 days    & DDBM      & 4 & 2 days   & 2 days \\
        SPDM+WT  & 2 & 2 weeks  & 2 weeks   & SPDM+FA   & 4 & 2 days   & 2 days \\
        SPDM+FA  & 2 & 5 days   & 5 days    & Pix2Pix   & 1 & 2 hours  & 1 day  \\
        SP-GAN   & 1 & 2 days   & 2 days    & I\textsuperscript{2}SB & 2 & 3 days   & 3 days \\
        GE-GAN   & 1 & 2 days   & 2 days    &           &   &          &        \\
        \bottomrule
    \end{tabular}
    }
    }{
        \caption{Model training times for experiments discussed in \cref{sec:emp_study}.}
        \label{tab:model_training_times}
    }
\end{table}

\section{FID COMPUTATION}
\label{appx:fid_computation}

In \cref{sec:results} we report the Fr\'echet intercept distance (FID) \citep{HeuselRHTH2017} score of the various models on the datasets described in \cref{sec:emp_study} and \cref{appx:datasets}, respectively under $C_4$ and $D_4$ groups. In order to make the FID score robust to changes in image orientation, meaning the features the underlying {\rm InceptionV3} model extracts from the reference dataset can be compared to those extracted from the generated samples, we average the reference statistics over all actions within the group considered. In particular, suppose $\mathcal{D}_{ref}$ is a reference dataset (e.g., all the images within the rotated MNIST dataset) and $\mathcal{D}_{s}$ is a collection of samples generated from the model being evaluated with respect to group $\cal G$. Let ${\rm T}(\cdot)$ denote the operation that returns the mean and covariance statistics of the features extracted from a dataset; i.e., $T(\mathcal{D_s})=(\mu_s, \Sigma_s)$. Moreover, recall that FID is computed using the expression:
\begin{equation}
    {\rm FID} = d^2({\rm T}(\mathcal{D}_{ref},{\rm T}(\mathcal{D}_s)) = \|\mu_{ref}-\mu_s\|^2_2 + {\rm Tr}(\Sigma_{ref}+\Sigma_s-2(\Sigma_{ref}\Sigma_s)^{1/2}).
\end{equation}
Then the calculation of the FID with respect to the group $\mathcal{G}$ is done by first computing 
\begin{equation}
    {\rm T}_\mathcal{G}(\mathcal{D})=\frac{1}{|\mathcal{D}|}\sum_{\hs \in \mathcal{G}}T(A_\hs \mathcal{D}) = (\hat\mu, \hat\Sigma),
\end{equation}
where $A_\hs \mathcal{D} = \{A_\hs \xs \mid \xs\in \mathcal{D}\}$. Then we compute the final FID score as 
\begin{equation}
    {\rm FID}_\mathcal{G} = d^2({\rm T}_\mathcal{G}(\mathcal{D}_{ref}), {\rm T}(\mathcal{D}_s)).
\end{equation}
This formulation ensures that the reference statistics used in computing the FID score of a model conditioned to be equivalent are not biased. All FID values reported in \cref{tab:fid-rot-mnist} and \cref{tab:bridge_models}, potentially excluding those reported by other authors, were calculated in this fashion.

\section{LIMITATIONS}
\label{appx:limitations}

Here we provide a summary discussion of some limitations of the proposed method. Note that although our theory can handle arbitrary groups of linear isometries, our implemented methods are restricted to groups with finitely many elements, as the techniques +WT and +FA proposed cannot be immediately generalized to infinite groups without necessary approximation. With that said, we would like to note that when a specific group is chosen it may be feasible to design specialized models to achieve perfect equivariance even if the group contains infinite elements. For example, when not considering the structure of the drift’s attribute, existing work typically employs GNN to achieve SO(3) or SE(3) equivariance. We believe this approach can also be adapted to fit our more general setting. As our work concentrates on general groups, we have decided to reserve the study of specific groups for future research.
\newpage
\section{SAMPLE IMAGES}
\label{appx:sample-images}
Here, we include a collection of generated image samples from the models discussed within the paper across the various datasets in Section~\ref{sec:emp_study}.

\begin{figure}[!ht]
\centering
  \begin{subfigure}[t]{0.45\columnwidth}
    \centering
    \includegraphics[width=\columnwidth]{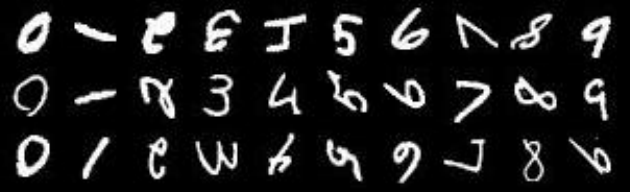}
    \caption{Reference $C_4$ rotated MNIST images.}
  \end{subfigure}
  ~
  \begin{subfigure}[t]{0.45\columnwidth}
    \centering
    \includegraphics[width=\columnwidth]{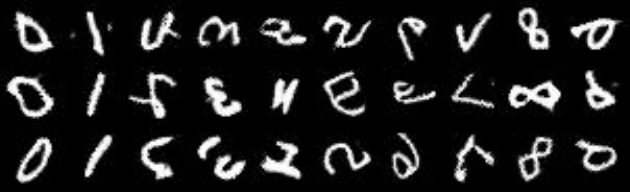}
    \caption{Images generated from SP-GAN.}
  \end{subfigure}
  \\
  \begin{subfigure}[t]{0.45\columnwidth}
    \centering
    \includegraphics[width=\columnwidth]{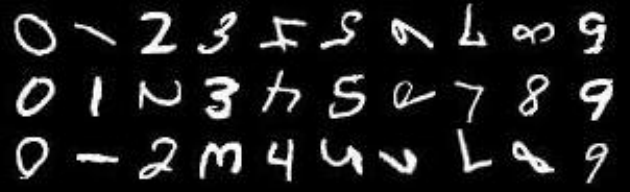}
    \caption{Images generated from VP-SDE.}
  \end{subfigure}
  ~
  \begin{subfigure}[t]{0.45\columnwidth}
    \centering
    \includegraphics[width=\columnwidth]{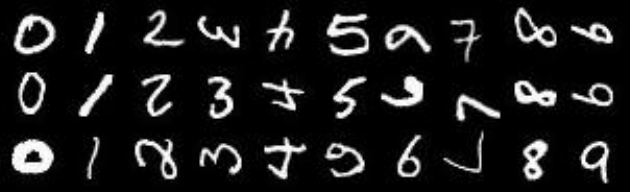}
    \caption{Images generated from SPDM+WT}
  \end{subfigure}
  \\
  \begin{subfigure}[t]{0.45\columnwidth}
    \centering
    \includegraphics[width=\columnwidth]{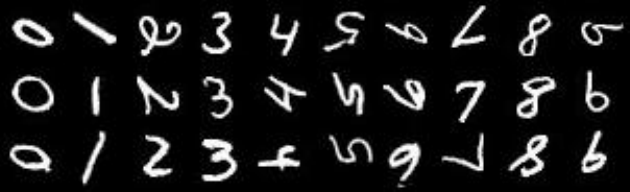}
    \caption{Images generated from SPDM+FA}
  \end{subfigure}
  ~
  \caption{Sample comparison between models trained on the Rotated MNIST 28x28x1 dataset as described in \cref{sec:emp_study}.}
  \label{fig:appx-lysto-images}
\end{figure}

\begin{figure}[!ht]
\centering
    \begin{subfigure}[t]{0.7\columnwidth}
        \centering
        \includegraphics[width=\columnwidth]{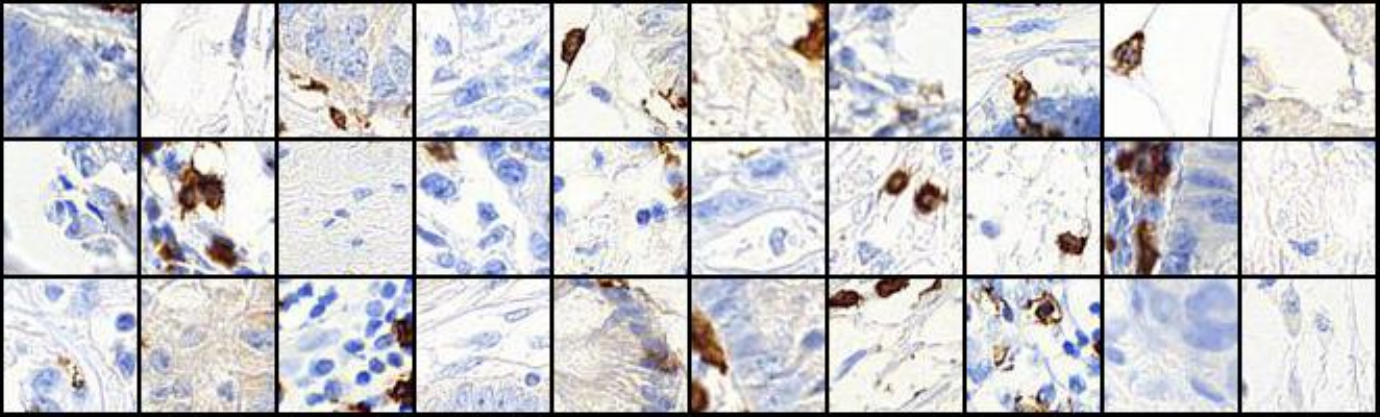}
        \caption{Reference images.}
    \end{subfigure}

    \begin{subfigure}[t]{0.7\columnwidth}
        \centering
        \includegraphics[width=\columnwidth]{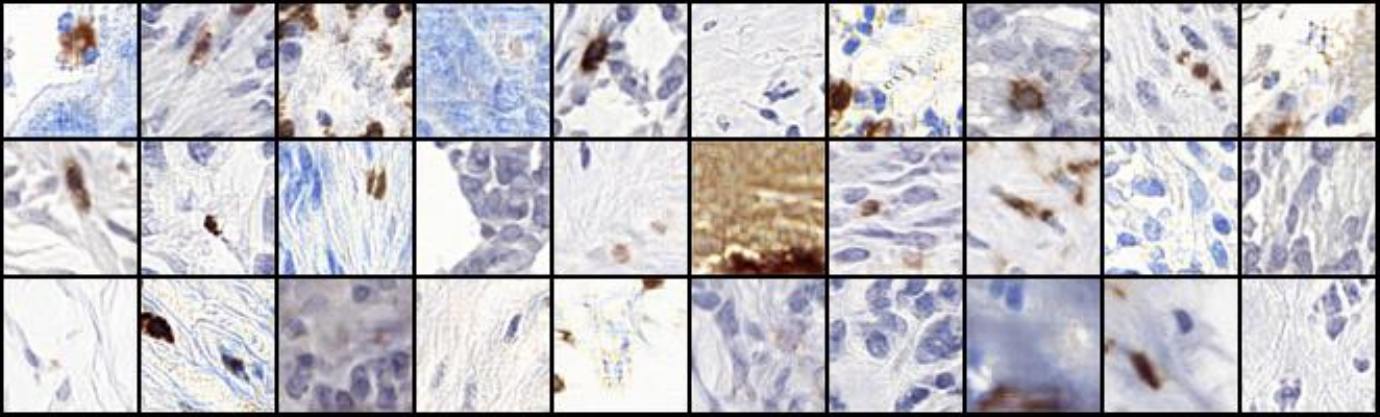}
        \caption{Images generated from SP-GAN.}
    \end{subfigure}

    \begin{subfigure}[t]{0.7\columnwidth}
        \centering
        \includegraphics[width=\columnwidth]{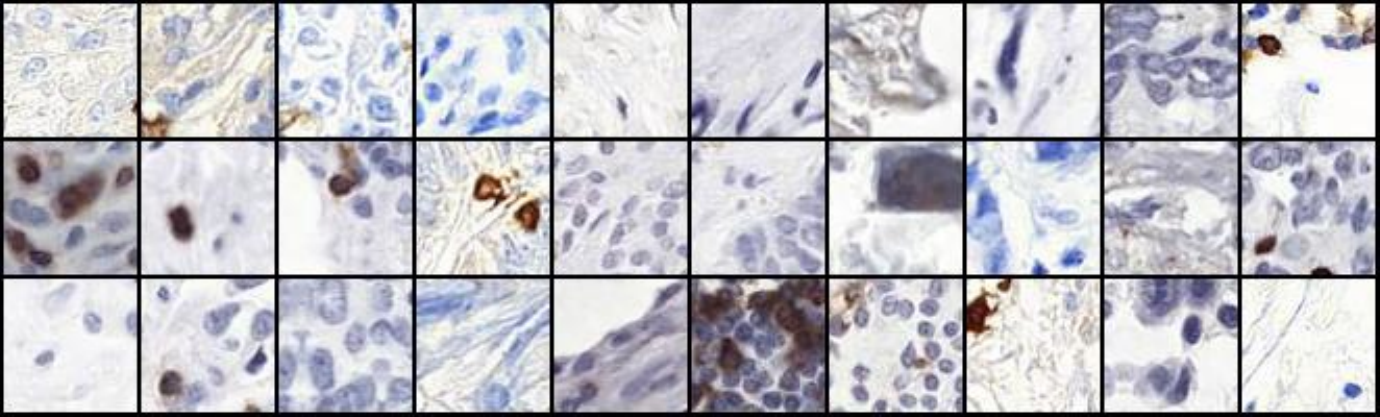}
        \caption{Images generated from VP-SDE}
    \end{subfigure}

    \begin{subfigure}[t]{0.7\columnwidth}
        \centering
        \includegraphics[width=\columnwidth]{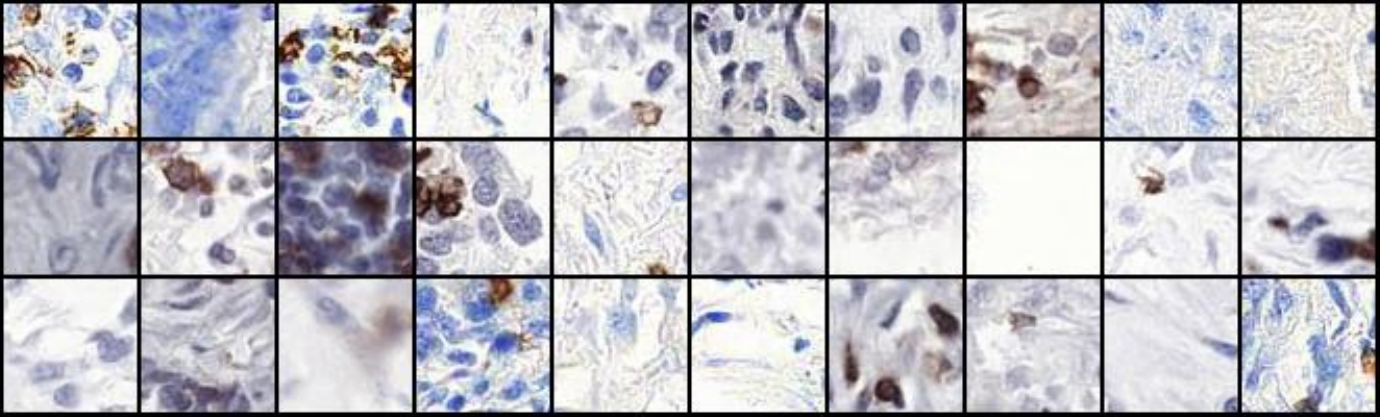}
        \caption{Images generated from SPDM+WT.}
    \end{subfigure}
    
    \begin{subfigure}[t]{0.7\columnwidth}
        \centering
        \includegraphics[width=\columnwidth]{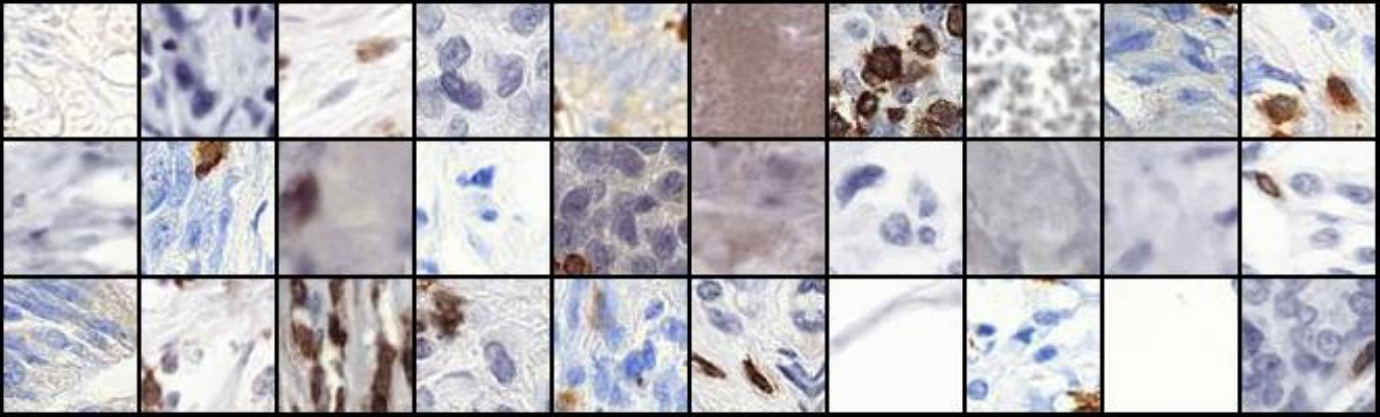}
        \caption{Images generated from SPDM+FA}
    \end{subfigure}

    \caption{Sample comparison between models trained on the LYSTO 64x64x3 dataset from \cref{sec:emp_study}.}
\end{figure}

    
    

\begin{figure}[!ht]
\centering
    \begin{subfigure}[t]{0.7\columnwidth}
        \centering
        \includegraphics[width=\columnwidth]{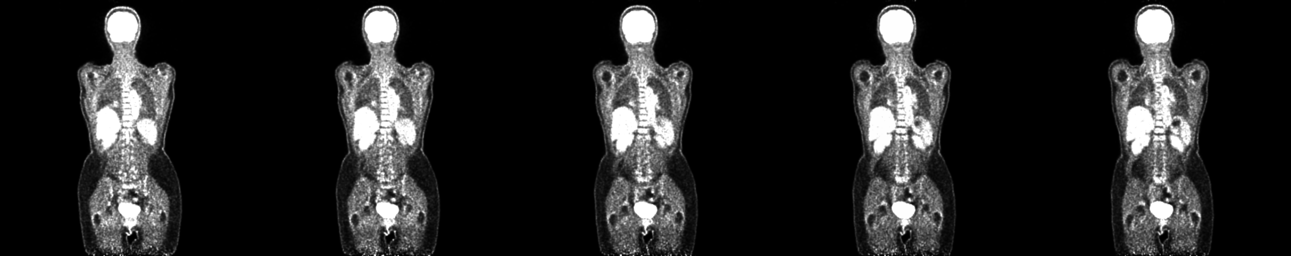}
    \end{subfigure}
    \begin{subfigure}[t]{0.7\columnwidth}
        \centering
        \includegraphics[width=\columnwidth]{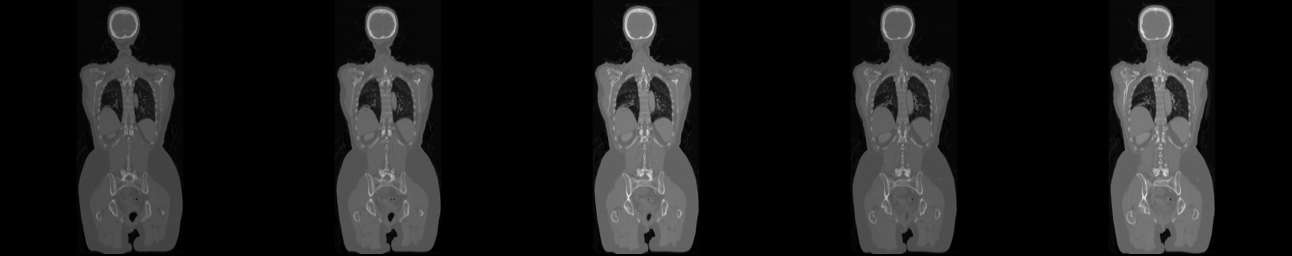}
        \caption{Reference images of CT and PET images.}
    \end{subfigure}

    \begin{subfigure}[t]{0.7\columnwidth}
        \centering
        \includegraphics[width=\columnwidth]{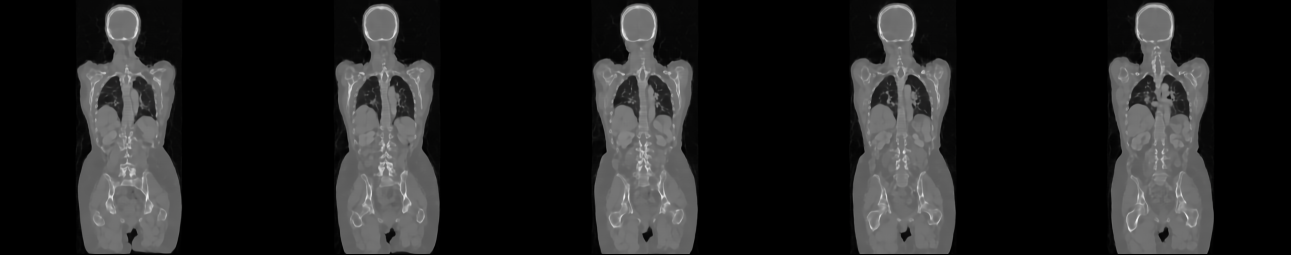}
        \caption{Images generated from DDBM.}
    \end{subfigure}

    \begin{subfigure}[t]{0.7\columnwidth}
        \centering
        \includegraphics[width=\columnwidth]{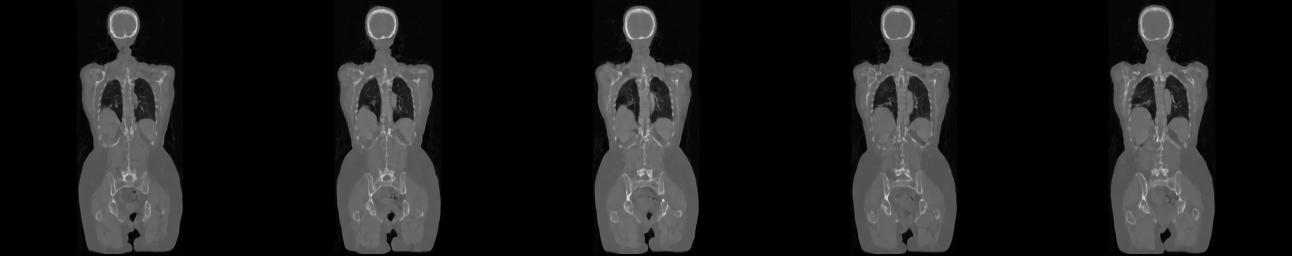}
        \caption{Images generated from SPDM+FA}
    \end{subfigure}

    \begin{subfigure}[t]{0.7\columnwidth}
        \centering
        \includegraphics[width=\columnwidth]{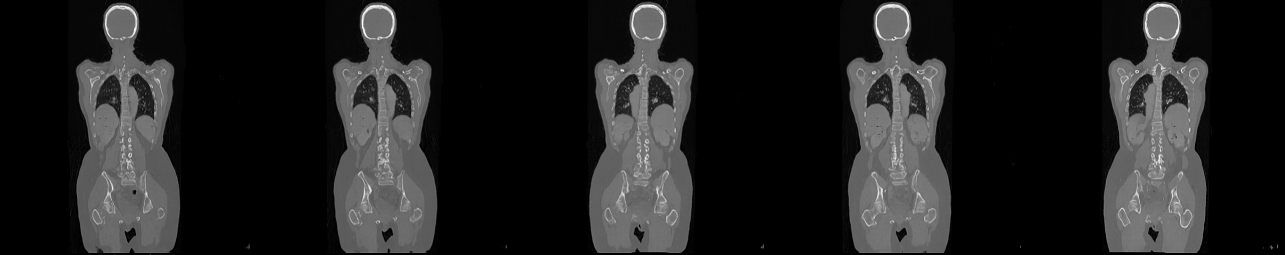}
        \caption{Images generated from Pix2Pix.}
    \end{subfigure}
    
    \begin{subfigure}[t]{0.7\columnwidth}
        \centering
        \includegraphics[width=\columnwidth]{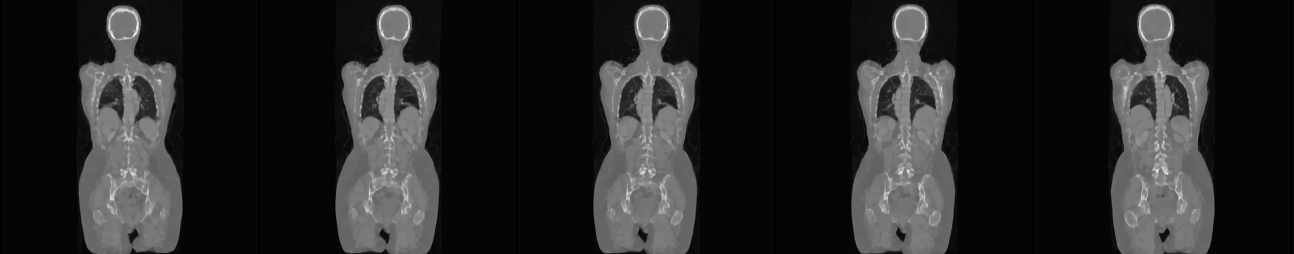}
        \caption{Images generated from ${\rm I^2SB}$}
    \end{subfigure}

    \caption{Sample comparison between models trained on the CT-PET 256x256x3 dataset from \cref{sec:emp_study} and \cref{appx:datasets}.}
\end{figure}
\FloatBarrier

\end{document}